\documentclass{article}

\usepackage[final]{nips_2018}

\usepackage{tabularx}
\usepackage{arydshln}
\usepackage[utf8]{inputenc} %
\usepackage[T1]{fontenc}    %
\usepackage{hyperref}       %
\usepackage{url}            %
\usepackage{booktabs}       %
\usepackage{amsthm}
\usepackage{amsmath}
\usepackage{amsfonts}
\usepackage{nicefrac}       %
\usepackage{microtype}      %
\usepackage{wrapfig}
\usepackage{subcaption}

\usepackage{tikz}
\usetikzlibrary{arrows,automata}
\usepackage{array,multirow,graphicx}

\usepackage{algorithm}

\usepackage{algorithmic}
\title{Variational Inverse Control with Events: A General Framework for Data-Driven Reward Definition}

\newcommand{\bX}{\textbf{X}}
\newcommand{\bY}{\textbf{Y}}

\newcommand{\Or}{\textrm{ or }}

\newtheorem{theorem}{Theorem}[section]

\newtheorem{lemma}{Lemma}[section]
\newcommand{\repeatthanks}{\textsuperscript{\thefootnote}}

\author{
  Justin Fu\thanks{equal contribution} \quad Avi Singh\repeatthanks \quad Dibya Ghosh \quad Larry Yang \quad Sergey Levine \\
  University of California, Berkeley\\
  \texttt{\{justinfu, avisingh, dibyaghosh, larrywyang, svlevine\}@berkeley.edu} \\
}

\begin{document}

\maketitle

\begin{abstract}
The design of a reward function often poses a major practical challenge to real-world applications of reinforcement learning. Approaches such as inverse reinforcement learning attempt to overcome this challenge, but require expert demonstrations, which can be difficult or expensive to obtain in practice.
We propose variational inverse control with events (VICE), which generalizes inverse reinforcement learning methods to cases where full demonstrations are not needed, such as when only samples of desired goal states are available. Our method is grounded in an alternative perspective on control and reinforcement learning, where an agent's goal is to maximize the probability that one or more events will happen at some point in the future, rather than maximizing cumulative rewards. We demonstrate the effectiveness of our methods on continuous control tasks, with a focus on high-dimensional observations like images where rewards are hard or even impossible to specify.
\end{abstract}

\section{Introduction}

Reinforcement learning (RL) has shown remarkable promise in recent years, with results on a range of complex tasks such as robotic control~\citep{Levine16} and playing video games~\citep{Mnih2015} from raw sensory input. 
RL algorithms solve these problems by learning a policy that maximizes a reward function that is considered as part of the problem formulation. There is little practical guidance that is provided in the theory of RL about how these rewards should be designed. However, the design of the reward function is in practice critical for good results, and reward misspecification can easily cause unintended behavior~\citep{Amodei16}. For example, a vacuum cleaner robot rewarded to pick up dirt could exploit the reward by repeatedly dumping dirt on the ground and picking it up again~\citep{RussellNorvigAI}. Additionally, it is often difficult to write down a reward function at all. For example, when learning policies from high-dimensional visual observations, practitioners often resort to using motion capture~\citep{jason17} or specialized computer vision systems~\citep{progressive} to obtain rewards.

As an alternative to reward specification, imitation learning~\citep{Argall09} and inverse reinforcement learning \citep{Ng2000} instead seek to mimic expert behavior. However, such approaches require an expert to show \emph{how} to solve a task. We instead propose a novel problem formulation, variational inverse  control with events (VICE), which generalizes inverse reinforcement learning to alternative forms of expert supervision. In particular, we consider cases when we have examples of a desired final outcome, rather than full demonstrations, so the expert only needs to show \emph{what} the desired outcome of a task is (see Figure ~\ref{fig:splash}). A straightforward way to make use of these desired outcomes is to train a classifier~\citep{Pinto16,Tung18} to distinguish desired and undesired states. However, for these approaches it is unclear how to correctly sample negatives and whether using this classifier as a reward will result in the intended behavior, since an RL agent can learn to exploit the classifier, in the same way it can exploit human-designed rewards. Our framework provides a more principled approach, where classifier training corresponds to learning probabilistic graphical model parameters (see Figure~\ref{fig:splash_2}), and policy optimization corresponds to inferring the optimal actions. By selecting an inference query which corresponds to our intentions, we can mitigate reward hacking scenarios similar to those previously described, and also specify the task with examples rather than manual engineering.

\begin{wrapfigure}{r}{0.5\textwidth}
\centering
\vspace{-0.3cm}
\includegraphics[width=0.4\textwidth]{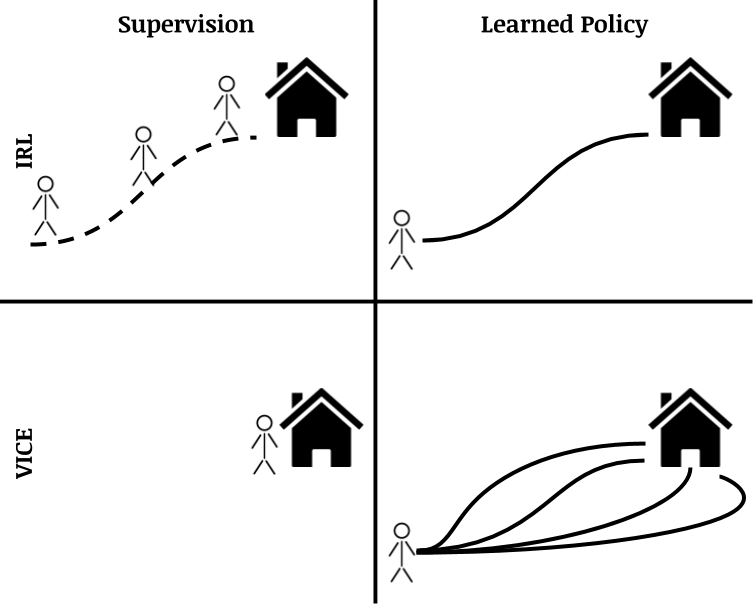}
\vspace{-0.3cm}
\caption{
\label{fig:splash} \small 
Standard IRL requires full expert demonstrations and aims to produce an agent that mimics the expert. VICE generalizes IRL to cases where we only observe final desired outcomes, which does not require the expert to actually know \emph{how} to perform the task.}
\vspace{-0.15in}
\end{wrapfigure}

Our inverse formulation is based on a corresponding forward control framework which reframes control as inference in a graphical model. Our framework resembles prior work~\citep{Kappen09,Toussaint09, Rawlik12}, but we extend this connection by replacing the conventional notion of rewards with event occurence variables. Rewards correspond to log-probabilities of events, and value functions can be interpreted as backward messages that represent log-probabilities of those events occurring. This framework retains the full expressivity of RL, since any rewards can be expressed as log-probabilities, while providing more intuitive guidance on task specification. It further allows us to express various intentions, such as for an event to happen at least once, exactly once at any time step, or once at a specific timestep. Crucially, our framework does not require the agent to \emph{observe} the event happening, but only to know the probability that it occurred. While this may seem unusual, it is more practical in the real world, where success may be determined by probabilistic models that themselves carry uncertainty. For example, the previously mentioned vacuum cleaner robot needs to estimate from its observations whether its task has been accomplished and would never receive direct feedback from the real world whether a room is clean. %

\begin{wrapfigure}{r}{0.5\textwidth}
\centering
\includegraphics[width=0.5\textwidth]{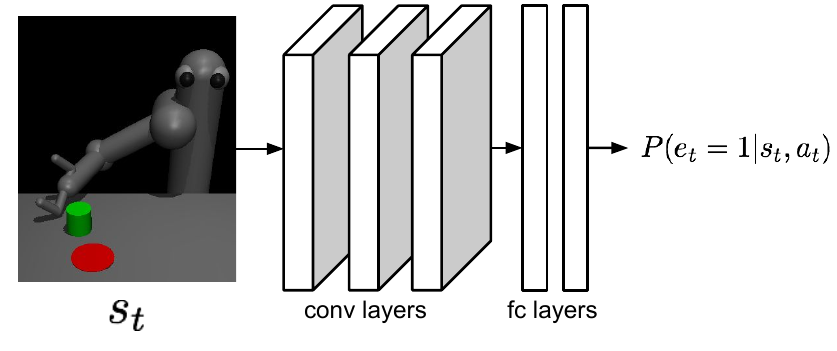}
\vspace{-0.6cm}
\caption{
\label{fig:splash_2} \small 
Our framework learns event probabilities from data. We use neural networks as function approximators to model this distribution, which allows us to work with high dimensional observations like images.}
\vspace{-0.20in}
\end{wrapfigure}

Our contributions are as follows. We first introduce the event-based control framework by extending previous control as inference work to alternative queries which we believe to be useful in practice. This view on control can ease the process of reward engineering by mapping a user's intention to a corresponding inference query in a probabilistic graphical model. Our experiments demonstrate how different queries can result in different behaviors which align with the corresponding intentions. 
We then propose methods to learn event probabilities from data, in a manner analogous to inverse reinforcement learning. This corresponds to the use case where designing event probabilities by hand is difficult, but observations (e.g., images) of successful task completion are easier to provide. This approach is substantially easier to apply in practical situations, since full demonstrations are not required. %
Our experiments demonstrate that our framework can be used in this fashion for policy learning from high dimensional visual observations where rewards are hard to specify. Moreover, our method substantially outperforms baselines such as sparse reward RL, indicating that our framework provides an automated shaping effect when learning events, making it feasible to solve otherwise hard tasks.

\section{Related work}
Our reformulation of RL is based on the connection between control and inference~\citep{Kappen09,Ziebart10,Rawlik12}. The resulting problem is sometimes referred to as maximum entropy reinforcement learning, or KL control. Duality between control and inference in the case of linear dynamical systems has been studied in ~\citet{Kalman60,Todorov08}. Maximum entropy objectives can be optimized efficiently and exactly in linearly solvable MDPs~\citep{Todorov07} and environments with discrete states. In linear-quadratic systems, control as inference techniques have been applied to solve path planning problems for robotics~\citep{Toussaint09}. In the context of deep RL, maximum entropy objectives have been used to derive soft variants of Q-learning and policy gradient algorithms~\citep{Haarnoja2017,Schulman17,ODonoghue16,Nachum17}. These methods embed the standard RL objective, formulated in terms of rewards, into the framework of probabilistic inference. In contrast, we aim specifically to reformulate RL in a way that does not require specifying arbitrary scalar-valued reward functions.

In addition to studying inference problems in a control setting, we also study the problem of learning event probabilities in these models. This is  related to prior work on inverse reinforcement learning (IRL), which has also sought to cast learning of objectives into the framework of probabilistic models~\citep{Ziebart08,Ziebart10}. As explained in Section~\ref{sec:learning_events}, our work generalizes IRL to cases where we only provide examples of a desired outcome or goal, which is significantly easier to provide in practice since we do not need to know how to achieve the goal.

Reward design is crucial for obtaining the desired behavior from RL agents~\citep{Amodei16}. \citet{Ng2000} showed that rewards can be modified, or shaped, to speed up learning without changing the optimal policy. \citet{Singh10} study the problem of optimal reward design, and introduce the concept of a fitness function. They observe that a proxy reward that is distinct from the fitness function might be optimal under certain settings, and \citet{Sorg10} study the problem of how this optimal proxy reward can be selected. \citet{Hadfield-Menell17} introduce the problem of inferring the true objective based on the given reward and MDP. Our framework aids task specification by introducing two decisions: the selection of the inference query that is of interest (i.e., when and how many times should the agent cause the event?), and the specification of the event of interest. Moreover, as discussed in Section~\ref{sec:experiments}, we observe that our method automatically provides a reward shaping effect, allowing us to solve otherwise hard tasks. %

\section{Preliminaries}

In this section we introduce our notation and summarize how control can be framed as inference. Reinforcement learning operates on Markov decision processes (MDP), defined by the tuple $(\mathcal{S}, \mathcal{A}, \mathcal{T}, r, \gamma, \rho_0)$. $\mathcal{S}, \mathcal{A}$ are the state and action spaces, respectively, $r$ is a reward function, which is typically taken to be a scalar field on $\mathcal{S} \times \mathcal{A}$, and $\gamma \in (0, 1)$ is the discount factor. $\mathcal{T}$ and $\rho_0$ represent the dynamics and initial state distributions, respectively. %

\subsection{Control as inference}
\label{sec:graphical_model_control}

\begin{wrapfigure}{r}{0.5\textwidth}

\vspace{-0.05in}

\centering
\begin{tikzpicture}[->,auto,node distance=1.3cm,]
\node[state](S1){$s_1$};
\node[state](S2)[right of=S1]{$s_2$};
\node[state](S3)[right of=S2]{$s_3$};
\node[state](A1)[below of=S1]{$a_1$};
\node[state](A2)[below of=S2]{$a_2$};
\node[state](A3)[below of=S3]{$a_3$};
\node[state](E1)[below of=A1,fill=gray!30]{$e_1$};
\node[state](E2)[below of=A2,fill=gray!30]{$e_2$};
\node[state](E3)[below of=A3,fill=gray!30]{$e_3$};
\node(DOTS)[right of=S3]{\ldots};
\node(DOTA)[right of=A3]{\ldots};
\node(DOTE)[right of=E3]{\ldots};
\node[state](ST)[right of=DOTS]{$s_T$};
\node[state](AT)[below of=ST]{$a_T$};
\node[state](ET)[below of=AT,fill=gray!30]{$e_T$};

\path (S1) edge node {} (S2)
           edge node {} (A1)
           edge [bend right] node {} (E1)
      (A1) edge node {} (S2)
           edge node {} (E1)
      (S2) edge node {} (S3)
           edge node {} (A2)
           edge [bend right] node {} (E2)
      (A2) edge node {} (S3)
           edge node {} (E2)
      (S3) edge node {} (A3)
           edge node {} (DOTS)
           edge [bend right] node {} (E3)
      (A3) edge node {} (E3)
           edge node {} (DOTS)
      (DOTS) edge node {} (ST)
      (DOTA) edge node {} (ST)
      (ST) edge node {} (AT)
           edge [bend right] node {} (ET)
      (AT) edge node {} (ET)
          ;
\end{tikzpicture}
\caption{ \small
\label{fig:controlgraph}
A graphical model framework for control. In maximum entropy reinforcement learning, we observe $e_{1:T}=1$ and can perform inference on the trajectory to obtain a policy.}
\vspace{-0.15in}
\end{wrapfigure}
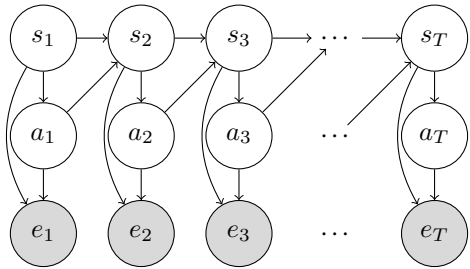

In order to cast control as an inference problem, we begin with the standard graphical model for an MDP, which consists of states and actions. We incorporate the notion of a goal with an additional variable $e_t$ that depends on the state (and possibly also the action) at time step $t$, according to $p(e_t | s_t, a_t)$. If the goal is specified with a reward function, we can define $p(e_t=1|s_t, a_t) = e^{r(s,a)}$ which, as we discuss below, leads to a maximum entropy version of the standard RL framework. This requires the rewards to be negative, which is not restrictive in practice, since if the rewards are bounded we can re-center them so that the maximum value is 0. The structure of this model is presented in Figure~\ref{fig:controlgraph}, and is also considered in prior work, as discussed in the previous section.

The maximum entropy reinforcement learning objective emerges when we condition on $e_{1:T}=1$. Consider computing a backward message $\beta(s_t, a_t) = p(e_{t:T}=1|s_t, a_t)$. Letting $Q(s_t, a_t) = \log \beta(s_t, a_t)$, notice that the backward messages encode the backup equations
\begin{align*}
Q(s_t, a_t) = r(s_t, a_t) + \log E_{s_{t+1}}[e^{V(s_{t+1})}]
\ \ \ \ \ \ \ \ \ \ \ \ 
V(s_t) = \log \int_{a\in\mathcal{A}} e^{Q(s_t,a)}da
\ \ .
\end{align*}
We include the full derivation in Appendix~\ref{app:event_derivations_mprl}, which resembles derivations discussed in prior work~\citep{Ziebart08}.
This backup equation corresponds to maximum entropy RL, and is equivalent to soft Q-learning and causal entropy RL formulations in the special case of deterministic dynamics~\citep{Haarnoja2017,Schulman17}. For the case of stochastic dynamics, maximum-entropy RL is optimistic with respect to the dynamics and produces risk-seeking behavior, and we refer the reader to Appendix~\ref{sec:variational_inf}, which covers a variational derivation of the policy objective which properly handles stochastic dynamics.  %

\section{Event-based control}
\label{sec:event_based_control}
In control as inference, we chose \mbox{$\log p(e_t=1|s_t,a_t) = r(s,a)$} so that the resulting inference problem matches the maximum entropy reinforcement learning objective. However, we might also ask: what does the variable $e_t$, and its probability, represent? The connection to graphical models lets us interpret rewards as the log-probability that an event occurs, and the standard approach to reward design can also be viewed as specifying the probability of some binary event, that we might call an optimality event. %
This provides us with an alternative way to think about task specification: rather than using arbitrary scalar fields as rewards, we can specify the events for which we would like to maximize the probability of occurrence.

We now outline inference procedures for different types of problems of interest in the graphical model depicted in Figure~\ref{fig:controlgraph}. In Section~\ref{sec:learning_events}, we will discuss learning procedures in this graphical model which allow us to specify objectives from data. The strength of the events framework for task specification lies in both its intuitive interpretation and flexibility: though we can obtain similar behavior in standard reinforcement learning, it may require considerable reward tuning and changes to the overall problem statement, including the dynamics. In contrast, events provides a single unified framework where the problem parameters remain unchanged, and we simply ask the appropriate queries. We will discuss:
\begin{itemize}
\item \textbf{ALL} query: $p(\tau | e_{1:T}=1)$, meaning the event should happen at each time step.
\item \textbf{AT} query: $p(\tau | e_{t^*}=1)$, meaning the event should happen at a specific time $t^*$.
\item \textbf{ANY} query: $p(\tau | e_1=1 \Or e_2=1 \Or ...\  e_T=1)$ meaning the event should happen on at least one time step during each trial.
\end{itemize}
We present two derivations for each query: a conceptually simple one based on maximum entropy and message passing (see Section~\ref{sec:graphical_model_control}), and one based on variational inference, (see Appendix~\ref{sec:variational_inf}), which is more appropriate for stochastic dynamics. The resulting variational objective is of the form: 
\[ J(\pi) = -D_{KL}(\pi(\tau)||p(\tau|\textrm{evidence})) = E_{s_{1:T},a_{1:T} \sim q}[\hat{Q}(s_{1:T},a_{1:T})+H^\pi(\cdot|s_{1:T}) ], \]
where $\hat{Q}$ is an empirical Q-value estimator for a trajectory and $H^\pi(\cdot|s_{1:T})= -\sum_{t=0}^T \log \pi(a_t|s_t)$ represents the entropy of the policy. This form of the objective can be used in policy gradient algorithms, and in special cases can also be written as a recursive backup equation for dynamic programming algorithms.
We directly present our results here, and present more detailed derivations (including extensions to discounted cases) in Appendices~\ref{app:event_derivations_mp} and ~\ref{app:event_derivations_causal}.

\subsection{ALL and AT queries}

We begin by reviewing the ALL query, when we wish for an agent to trigger an event at every timestep. This can be useful, for example, when expressing some continuous task such as maintaining some sort of configuration (such as balancing on a unicycle) or avoiding an adverse outcome, such as not causing an autonomous car to collide. As covered in Section~\ref{sec:graphical_model_control}, conditioning on the event at all time steps mathematically corresponds to the same problem as entropy maximizing RL, with the reward given by $\log p(e_t=1|s_t, a_t)$.

\begin{theorem}[ALL query]
In the ALL query, the message passing update for the Q-value can be written as:
\[
Q(s_t, a_t) = \log p(e_t=1|s_t, a_t) + \log E_{s_{t+1}}[e^{V(s_{t+1})}],
\]
where $Q(s_t, a_t)$ represents the log-message $\log p(e_{t:T}=1|s_t, a_t)$. The corresponding empirical Q-value can be written recursively as:
\[\hat{Q}(s_{t:T},a_{t:T}) = \log p(e_t=1|s_t, a_t) + \hat{Q}(s_{t+1:T}, a_{t+1:T}).\]
\end{theorem}
\begin{proof} 
See Appendices~\ref{app:event_derivations_mp_all} and ~\ref{app:event_derivations_causal_all}
\end{proof}

The AT query, or querying for the event at a specific time step, results in the same equations, except $\log p(e=1|s_t, a_t)$, is only given at the specified time $t^*$. %
While we generally believe that the ANY query presented in the following section will be more broadly applicable, there may be scenarios where an agent needs to be in a particular configuration or location at the end of an episode. In these cases, the AT query would be the most appropriate.

\subsection{ANY query}
\label{sec:query_any}

The ANY query specifies that an event should happen at least once before the end of an episode, without regard for when in particular it takes place. Unlike the ALL and AT queries, the ANY query does not correspond to entropy maximizing RL and requires a new backup equation. It is also in many cases more appropriate: if we would like an agent to accomplish some goal, we might not care when in particular that goal is accomplished, and we likely don't need it to accomplish it more than once. This query can be useful for specifying behaviors such as reaching a goal state, completion of a task, etc. Let the stopping time $t^* = \min\{t\ge 0 | e_t=1\}$ denote the first time that the event occurs.

\begin{theorem}[ANY query]
In the ANY query, the message passing update for the Q-value can be written as:
\begin{align*}
&Q(s_t, a_t) =
\log\left( p(e_t=1|s_t, a_t) + p(e_t=0|s_t,a_t)E_{s_{t+1}}[e^{V(s_{t+1})}] \right)
\end{align*}
where $Q(s_t, a_t)$ represents the log-message $\log p(t \le t^* \le T|s_t, a_t)$. The corresponding empirical Q-value can be written recursively as:
\begin{align*}
&\hat Q(s_{t:T}, a_{t:T}) =
\log\left( p(e_t=1|s_t, a_t) + p(e_t=0|s_t,a_t)e^{\hat{Q}(s_{t+1:T},a_{t+1:T})} \right).
\end{align*}

\end{theorem}
\begin{proof} 
See Appendices~\ref{app:event_derivations_mp_any} and ~\ref{app:event_derivations_causal_any}
\end{proof}

This query is related to first-exit RL problems, where an agent receives a reward of 1 when a specified goal is reached and is immediately moved to an absorbing state but it does not require the event to actually be observed, which makes it applicable to a variety of real-world situations that have uncertainty over the goal. The backup equations of the ANY query are equivalent to the first-exit problem when $p(e|s,a)$ is deterministic. This can be seen by setting $p(e=1|s,a)=r_{F}(s,a)$, where $r_F(s,a)$ is an goal indicator function that denotes the reward of the first-exit problem. In this case, we have $Q(s,a)=0$ if the goal is reachable, and $Q(s,a) = -\infty$ if not. In the first-exit case, we have $Q(s,a)=1$ if the goal is reachable and $Q(s,a)=0$ if not - both cases result in the same policy.

\subsection{Sample-based optimization using policy gradients}
\label{sec:pg}

In small, discrete settings with known dynamics, we can use the backup equations in the previous section to solve for optimal policies with dynamic programming. For large problems with unknown dynamics, we can also derive model-free analogues to these methods, and apply them to complex tasks with high-dimensional function approximators. One commonly used method is the policy gradient, and which we can derive via logarithmic differentiation as:
\begin{align*}
\nabla_\theta J(\theta) &= - \nabla_\theta D_{KL}(\pi_\theta(\tau)||p(\tau|\textrm{evidence})) \\
&= E_{s_{1:T},a_{1:T} \sim \pi_\theta}\left[\sum_{t=1}^T \nabla \log\pi_\theta(a_t|s_t) (\hat{Q}(s_{1:T},a_{1:T}) + H^\pi(\cdot|s_{t:T}))\right]
\end{align*}
Under certain assumptions we can replace $\hat{Q}(s_{1:T},a_{1:T})$ with $\hat{Q}(s_{t:T},a_{t:T})$ to obtain an estimator which only depends on future returns. See Appendix~\ref{app:policy_gradient} for further explanation.

This estimator can be integrated into standard policy gradient algorithms, such as TRPO~\cite{Schulman15}, to train expressive inference models using neural networks. Extensions of our approach to other RL methods with function approximation, such as Q-learning and approximate dynamic programming, can also be derived from the backup equations, though this is outside the scope of the present work.

\section{Learning event probabilities from data}
\label{sec:learning_events}

In the previous section, we presented a control framework that operates on events rather than reward functions, and discussed how the user can choose from among a variety of inference queries to obtain a desired outcome. However, the event probabilities must still be obtained in some way, and may be difficult to hand-engineer in many practical situations - for example, an image-based deep RL system may need an image classifier to determine if it has accomplished its goal. In such situations, we can ask the user to instead supply examples of states or observations where the event has happened, and learn the event probabilities $p_\theta(e=1|s,a)$. Inverse reinforcement learning corresponds to the case when we assume the expert triggers an event at all timesteps (the ALL query), in which case we require full demonstrations.
However, if we assume the expert is optimal under an ANY or AT query, full demonstrations are not required because the event is not assumed to be triggered at each timestep. This means our supervision can be of the form of a desired set of states 
rather than full trajectories. For example, in the vision-based robotics case, this means that we can specify goals using images of a desired goal state, which are much easier to obtain than full demonstrations. %without knowing a priori how to solve the task as an expert demonstration would require.

Formally, for each query, we assume our dataset of states and actions $(s, a) \sim p_{data}(s,a|e=1)$ when the event has happened, assuming the data-generating policy follows one of our inference queries. Our objective is imitation: we wish to train a model which produces samples that match the data. To that end, we learn the parameters of the model $p_\theta(s,a|e=1)$, trained with the maximum likelihood objective:
\[ \mathcal{L}(\theta) = -E_{p_{data}}\left[\log p_\theta(s,a|e=1)\right]\]

The gradient of this model is:
\begin{align}
\label{eqn:irl_grad}
\nabla_\theta \mathcal{L}(\theta) = -E_{p_{data}}\left[\nabla_\theta \log p_\theta(s,a|e=1) \right]
+ E_{p_{\theta}}\left[\nabla_\theta \log p_\theta(s,a|e=1)
\right]
\end{align}
Where the second term corresponds to the gradient of the partition function of $p_\theta(s, a |e=1)$. Thus, this implies an algorithm where we sample states and actions from the model $p_{\theta}$ and use them to compute the gradient update. 

%Note that we can always recover the event probabilities via Bayes rule: $p_\theta(s,a|\textrm{evidence}) \propto p_\theta(\textrm{evidence}|s,a)p(s,a) $, because $p(s,a)$ is defined to be the state-action marginal without any conditioning (which we can set to be produced by uniform policy, see Appendix~\ref{app:event_derivations_mprl})

\subsection{Sample-based optimization with discriminators}
\label{sec:airl}
In high-dimensional settings, a convenient method to perform the gradient update in Eqn.~\ref{eqn:irl_grad} is to embed the model $p_\theta(s,a|\textrm{evidence})$ within a discriminator between samples $p_\theta$ and data $p_{data}$ and take the gradient of the cross-entropy loss. Second, in order to draw samples from the model we instead train a "generator" policy via variational inference to draw samples from $p_{\theta}$. The variational inference procedure corresponds to those outlined in Section~\ref{sec:event_based_control}.

Specifically, we adapt the method of~\cite{Fu18}, which alternates between training a discriminator with the fixed form \[D_\theta(s,a) = p_\theta(s,a) / (p_\theta(s,a) + \pi(a|s))\] to distinguish between policy samples and success states, and a policy that minimizes the KL divergence between $D_{KL}(\pi(s,a) || p_\theta(s,a |=1))$. As shown in previous work~\citep{Finn16b,Fu18}, the gradient of the cross entropy loss of the discriminator is equivalent to the gradient of Eqn.~\ref{eqn:irl_grad}, and using the reward $\log D_\theta(s,a) - \log(1-D_\theta(s,a))$ with the appropriate inference objective is equivalent to minimizing KL between the sampler and generator. We show the latter equivalence in Appendix~\ref{app:irl}, and pseudocode for our algorithm is presented in Algorithm~\ref{alg:inverse_rl}

\begin{algorithm}[tb]
\caption{VICE: Variational Inverse Control with Events}
\label{alg:inverse_rl}
\begin{algorithmic}[1]
    \STATE Obtain examples of expert states and actions $s^E_i, a^E_i$
    \STATE Initialize policy $\pi$ and binary discriminator $D_{\theta}$.
    \FOR{step $n$ in \{1, \dots, N\}}
        \STATE Collect states and actions $s_i = (s_1, ..., s_T), a_i=(a_1, ..., a_T)$ by executing $\pi$.
        \STATE Train $D_{\theta}$ via logistic regression to classify expert data $s^E_i, a^E_i$ from samples $s_i, a_i$.
        \STATE Update $\pi$ with respect to $p_\theta$ using the appropriate inference objective.
    \ENDFOR
\end{algorithmic}
\end{algorithm}
\section{Experimental evaluation}
\label{sec:experiments}

Our experimental evaluation aims to answer the following questions: (1) How does the behavior of an agent change depending on the choice of query? We study this question in the case where the event probabilities are already specified. (2) Does our event learning framework (VICE) outperform simple alternatives, such as offline classifier training, when learning event probabilities from data? We study this question in settings where it is difficult to manually specify a reward function, such as when the agent receives raw image observations. (3) Does learning event probabilities provide better shaped rewards than the ground truth event occurrence indicators? Additional videos and supplementary material are available at \url{https://sites.google.com/view/inverse-event}.

\subsection{Inference with pre-specified event probabilities}

\begin{wrapfigure}{r}{0.4\textwidth}
\vspace{-0.7cm}
\centering 
\includegraphics[width=0.16\textwidth]{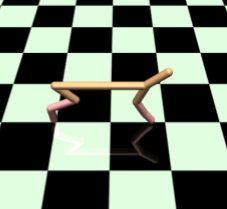}
\hspace{0.5cm}
\includegraphics[width=0.16\textwidth]{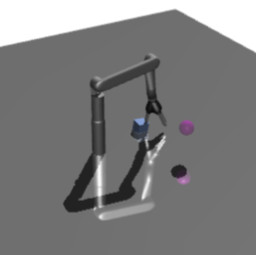}
\caption{\small \label{fig:goal_reaching_tasks}
 HalfCheetah and Lobber tasks.}
\vspace{-0.45cm}
\end{wrapfigure}

We first demonstrate how the ANY and ALL queries in our framework result in different behaviors. We adapt TRPO~\citep{Schulman15}, a natural policy gradient algorithm, to train policies using our query procedures derived in Section~\ref{sec:event_based_control}. Our examples involve two goal-reaching domains, HalfCheetah and Lobber, shown in Figure~\ref{fig:goal_reaching_tasks}. The goal of HalfCheetah is to navigate a 6-DoF agent to a goal position, and in Lobber, a robotic arm must throw an block to a goal position. To study the inference process in isolation, we manually design the event probabilities as $e^{-||x_{agent} - x_{target}||_2}$ for the HalfCheetah and $e^{- \|x_{block}-x_{goal}\|^2}$ for the Lobber.

\begin{wraptable}{r}{0.49\textwidth}
\small
\vspace{-0.6cm}
\begin{center}

\begin{tabular}{c|cc}
\hline
Query & Avg. Dist & Min. Dist \\
\hline
HalfCheetah-ANY & 1.35 (0.20) &\textbf{0.97} (0.46) \\
HalfCheetah-ALL & \textbf{1.33} (0.16) & 2.01 (0.48) \\
\hdashline
HalfCheetah-Random & 8.95 (5.37) & 5.41 (2.67) \\
\hline 
Lobber-ANY & 0.61 (0.12) &\textbf{0.25} (0.20) \\
Lobber-ALL & \textbf{0.59} (0.11) & 0.36 (0.21) \\
\hdashline
Lobber-Random & 0.93 (0.01) & 0.91 (0.01) \\
\hline 
\end{tabular}

\caption{\label{tbl:halfcheetah_results} \small 
 Results on HalfCheetah and Lobber tasks (5 trials). The ALL query generally results in superior returns, but the ANY query results in the agent reaching the target more accurately. Random refers to a random gaussian policy.}

\end{center}
\vspace{-0.6cm}
\end{wraptable}

The experimental results are shown in Table~\ref{tbl:halfcheetah_results}. While the average distance to the goal for both queries was roughly the same, the ANY query results in a much closer minimum distance. This makes sense, since in the ALL query the agent is punished for every time step it is not near the goal. The ANY query can afford to receive lower cumulative returns and instead has max-seeking behavior which more accurately reaches the target. Here, the ANY query better expresses our intention of reaching a target. 

\subsection{Learning event probabilities}

\begin{wraptable}{r}{0.6\textwidth}
\small
\vspace{-0.6cm}

\caption{\small Results on Maze, Ant and Pusher 
environments (5 trials). The metric reported is the final distance to the goal state (lower is better). VICE performs better than the classifier-based setup on all the tasks, and the performance is substantially better for the Ant and Pusher task. Detailed learning curves are provided in Appendix ~\ref{app:experiments}.}
\label{tbl:antvpm}
 \centering
\begin{tabular}{c|c|c c|c}
 \hline
 & \multicolumn{1}{c|}{Query type} & \multicolumn{1}{c}{Classifier} & \multicolumn{1}{c|}{VICE (ours)} & \multicolumn{1}{c}{True Binary} \\
 \hline
 \parbox[t]{2mm}{\multirow{2}{*}{\rotatebox[origin=c]{90}{\small{Maze}}}} 
 & ALL & 0.35 (0.29) & {\bf 0.20} (0.19) & {\multirow{2}{*}{0.11 (0.01)}}\\
 & ANY & 0.37 (0.21) & 0.23 (0.15) & \\
\hline

 \parbox[t]{2mm}{\multirow{2}{*}{\rotatebox[origin=c]{90}{Ant}}}
 & ALL & 2.71 (0.75) & 0.64 (0.32) & {\multirow{2}{*}{1.61 (1.35)}} \\
 & ANY & 3.93 (1.56) & {\bf 0.62} (0.55) &\\
\hline

 \parbox[t]{2mm}{\multirow{2}{*}{\rotatebox[origin=c]{90}{Push}}}
 & ALL & 0.25 (0.01) & {\bf 0.09 } (0.01) & {\multirow{2}{*}{0.17 (0.03)}} \\
 & ANY & 0.25 (0.01) &  0.11 (0.01) &\\
 
 \hline
 \end{tabular}
 \vspace{-0.4cm}
 
 \end{wraptable}

We now compare our event probability learning framework, which we call variational inverse control with events (VICE), against an offline classifier training baseline. We also compare our method to learning from true binary event indicators, to see if our method can provide some reward shaping benefits to speed up the learning process.
The data for learning event probabilities comes from success states.
That is, we have access to a set of states $\{s^{E}_i\}_{i=1...n}$, which may have been provided by the user, for which we know the event took place. This setting generalizes IRL, where instead of entire expert demonstrations, we simply have examples of successful states. The offline classifier baseline trains a neural network to distinguish success state ("positives") from states collected by a random policy. The number of positives and negatives in this procedure is kept balanced. This baseline is a reasonable and straightforward method to specify rewards in the standard RL framework, and provides a natural point of comparison to our approach, which can also be viewed as learning a classifier, but within the principled framework of control as inference. We evaluate these methods on the following tasks:

\textbf{Maze from pixels.} In this task, a point mass needs to navigate to a goal location through a small maze, depicted in Figure~\ref{fig:tasks}. The observations consist of 64x64 RGB images that correspond to an overhead view of the maze. The action space consists of X and Y forces on the robot. We use CNNs to represent the policy and the event distributions, training with 1000 success states as supervision.

\textbf{Ant.} In this task, a quadrupedal ``ant'' (shown in Figure~\ref{fig:tasks}) needs to crawl to a goal location, placed 3m away from its starting position. The state space contains joint angles and XYZ-coordinates of the ant. The action space corresponds to joint torques. We use 500 success states as supervision. 

\textbf{Pusher from pixels.} In this task, a 7-DoF robotic arm (shown in Figure~\ref{fig:tasks}) must push a cylinder object to a goal location. The state space contains joint angles, joint velocities and 64x64 RGB images, and the action space corresponds to joint torques. We use 10K success states as supervision.

\begin{wrapfigure}{r}{0.5\textwidth}
\vspace{-0.5cm}
\centering
\includegraphics[width=0.5\columnwidth]{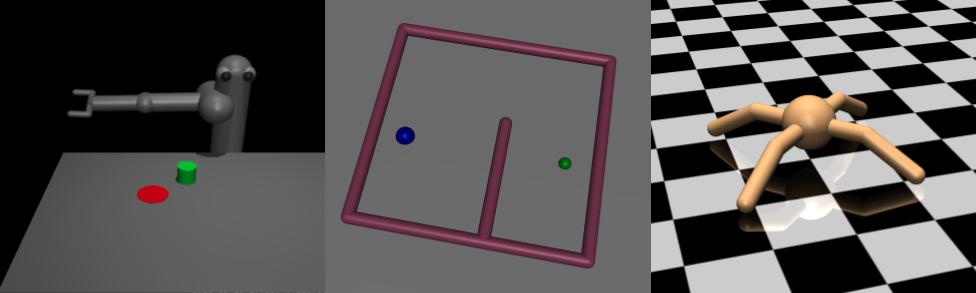}
\caption{ \small \label{fig:tasks}
Visualizations of the Pusher, Maze, and Ant tasks. In the Maze and Ant tasks, the agent seeks to reach a pre-specified goal position. In the Pusher task, the agent seeks to place a block at the goal position.}
\vspace{-0.5cm}
\end{wrapfigure}

Training details and neural net architectures can be found in Appendix ~\ref{app:experiments}. We also compare our method against a reinforcement learning baseline that has access to the true binary event indicator. For all the tasks, we define a ``goal region'', and give the agent a +1 reward when it is in the goal region, and 0 otherwise. Note that this RL baseline, which is similar to vanilla RL from sparse rewards, ``observes'' the event, providing it with additional information, while our model only uses the event probabilities learned from the success examples and receives no other supervision. It is included to provide a reference point on the difficulty of the tasks. Results are summarized in Table~\ref{tbl:antvpm}, and detailed learning curves can be seen in Figure~\ref{fig:push_curve} and Appendix~\ref{app:experiments}. We note the following salient points from these experiments. 

\begin{wrapfigure}{r}{0.5\textwidth}
\vspace{-0.5cm}
\centering
\includegraphics[width=0.4\columnwidth]{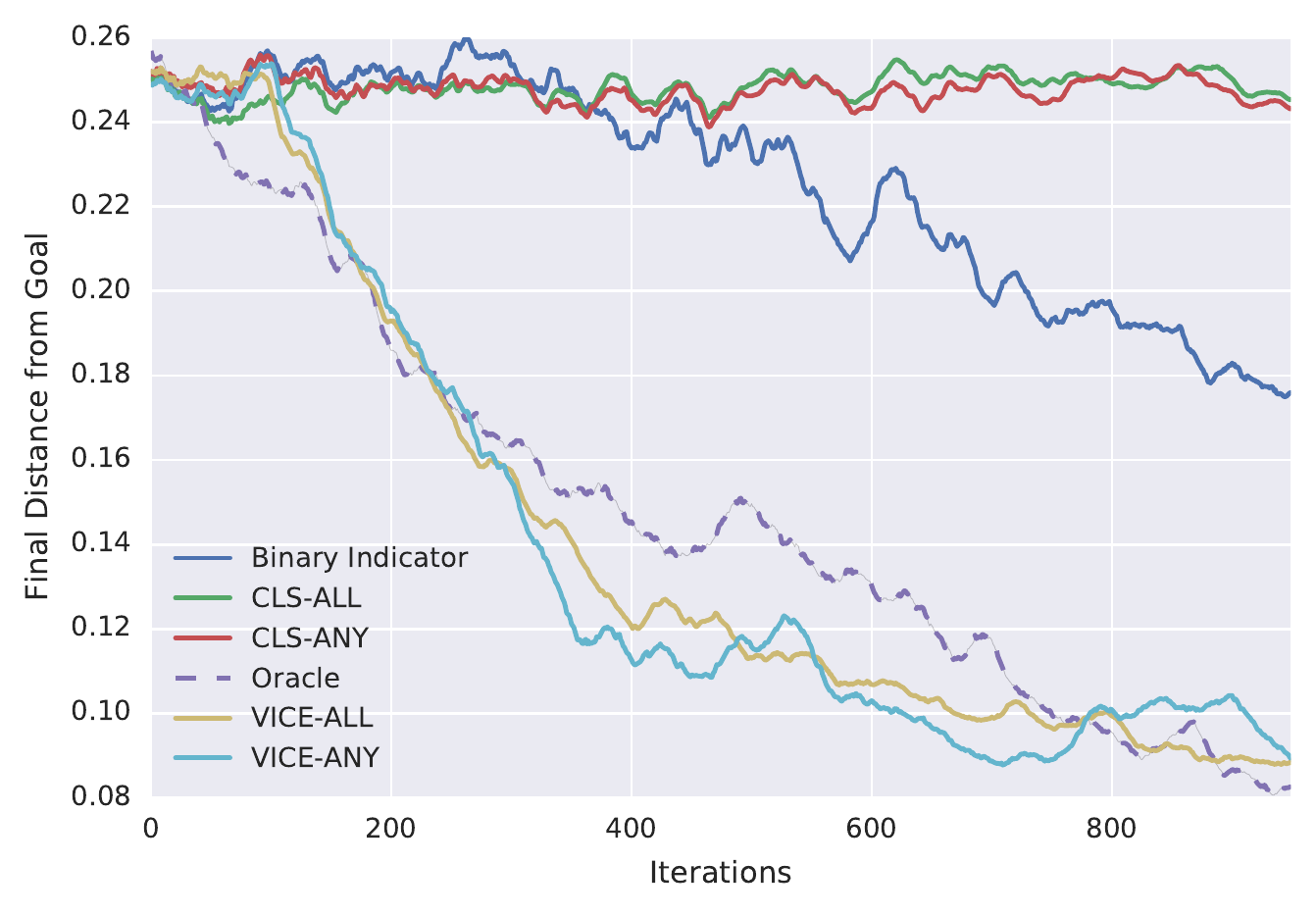}
\caption{ \label{fig:push_curve} \small  Results on the Pusher task (lower is better), averaged across five random seeds. VICE significantly outperforms the naive classifier and true binary event indicators. Further, the performance is comparable to learning from an oracle hand-engineered reward (denoted in dashed lines). Curves for the Ant and Maze tasks can be seen in Appendix~\ref{app:experiments}. }
\vspace{-0.5cm}
\end{wrapfigure}

\textbf{VICE outperforms na\"{i}ve classifier.} We observe that for \textit{Maze}, both the simple classifier and our method (VICE) perform well, though VICE achieves lower final distance. In the \textit{Ant} environment, VICE is crucial for obtaining good performance, and the simple classifier fails to solve the task. Similarly, for the \emph{Pusher} task, VICE significantly outperforms the classifier (which fails to solve the task). Unlike the na\"{i}ve classifier approach, VICE actively integrates negative examples from the current policy into the learning process, and appropriately models the event probabilities together with the dynamical properties of the task, analogously to IRL.

\textbf{Shaping effect of VICE.}
For the more difficult ant and pusher domains, VICE actually outperforms RL with the true event indicators. We analyze this shaping effect further in Figure~\ref{fig:push_curve}: our framework obtains performance that is superior to learning with true event indicators, while requiring much weaker supervision. This indicates that the event probability distribution learned by our method has a reward-shaping effect, which greatly simplifies the policy search process. We further compare our method against a hand-engineered shaped reward, depicted in dashed lines in Figure~\ref{fig:push_curve}. The engineered reward is given by $  - 0.2*\|x_{block}-x_{arm}\| - \|x_{block}-x_{goal}\|$, and is impossible to compute when we don't have access to $x_{block}$, which is usually the case when learning in the real world. We observe that our method achieves performance that is comparable to this engineered reward, indicating that our automated shaping effect is comparable to hand-engineered shaped rewards.

\section{Conclusion}

In this paper, we described how the connection between control and inference can be extended to derive a reinforcement learning framework that dispenses with the conventional notion of rewards, and replaces them with events. Events have associated probabilities. which can either be provided by the user, or learned from data. Recasting reinforcement learning into the event-based framework allows us to express various goals as different inference queries in the corresponding graphical model. The case where we learn event probabilities corresponds to a generalization of IRL where rather than assuming access to expert demonstrations, we assume access to states and actions where an event occurs. IRL corresponds to the case where we assume the event happens at every timestep, and we extend this notion to alternate graphical model queries where events may happen at a single timestep. 

\section*{Acknowledgements}
This research was supported by an ONR Young Investigator Program award, the National Science Foundation through IIS-1651843, IIS-1614653, and IIS-1700696, Berkeley DeepDrive, and donations from Google, Amazon, and NVIDIA.

\bibliography{example_paper}
\bibliographystyle{icml2018}

\newpage
\appendix
\part*{Appendices}

\section{Message Passing Updates for Reinforcement Learning}
\label{app:event_derivations_mprl}
In this section, we derive message passing updates that can be used to obtain an optimal policy in the graphical model for control (visualized below).

\begin{figure}[h]
\centering
\begin{tikzpicture}[->,auto,node distance=1.6cm]
\node[state](S1){$s_1$};
\node[state](S2)[right of=S1]{$s_2$};
\node[state](S3)[right of=S2]{$s_3$};
\node[state](A1)[below of=S1]{$a_1$};
\node[state](A2)[below of=S2]{$a_2$};
\node[state](A3)[below of=S3]{$a_3$};
\node[state](E1)[below of=A1,fill=gray!30]{$e_1$};
\node[state](E2)[below of=A2,fill=gray!30]{$e_2$};
\node[state](E3)[below of=A3,fill=gray!30]{$e_3$};
\node(DOTS)[right of=S3]{\ldots};
\node(DOTA)[right of=A3]{\ldots};
\node(DOTE)[right of=E3]{\ldots};
\node[state](ST)[right of=DOTS]{$s_T$};
\node[state](AT)[below of=ST]{$a_T$};
\node[state](ET)[below of=AT,fill=gray!30]{$e_T$};

\path (S1) edge node {} (S2)
           edge node {} (A1)
           edge [bend right] node {} (E1)
      (A1) edge node {} (S2)
           edge node {} (E1)
      (S2) edge node {} (S3)
           edge node {} (A2)
           edge [bend right] node {} (E2)
      (A2) edge node {} (S3)
           edge node {} (E2)
      (S3) edge node {} (A3)
           edge node {} (DOTS)
           edge [bend right] node {} (E3)
      (A3) edge node {} (E3)
           edge node {} (DOTS)
      (DOTS) edge node {} (ST)
      (DOTA) edge node {} (ST)
      (ST) edge node {} (AT)
           edge [bend right] node {} (ET)
      (AT) edge node {} (ET)
          ;
\end{tikzpicture}
\end{figure}

We define two backward messages, a state-action message $\beta(s_t, a_t) = p(e_{t:T}=1|s_t,a_t)$ and a state message $\beta(s_t) = p(e_{t:T}=1|s_t)$. The state message can be expanded in terms of the state-action message as:
\[
\beta(s_t) = p(e_{t:T}=1|s_t) = \int_{\mathcal{A}} \beta(s_t,a_t) p(a_t|s_t) da_t
\]

We can then write a recursive form for the state-action message in terms of the state message:
\begin{align*}
\beta(s_t, a_t) &= p(e_{t:T}=1|s_t,a_t)
= p(e_t=1|s_t,a_t)p(e_{t+1:T}=1|s_t,a_t)\\
&= p(e_t=1|s_t,a_t)\int_{\mathcal{S},\mathcal{A}} \beta(s_{t+1},a_{t+1})p(a_{t+1}|s_{t+1})p(s_{t+1}|s_t,a_t)
ds_{t+1}da_{t+1}\\
&= p(e_t=1|s_t,a_t)E_{s_{t+1}}[\beta(s_{t+1})]
\end{align*}

Next, we define $\log p(e_t=1|s_t,a_t)=r(s_t,a_t)$ as the reward factor and set the reference policy $p(a_t|s_t)=C$ to the uniform distribution as before. Non-uniform reference policies correspond to policy optimization with a modified reward function $r^{new}(s,a)= r^{old}(s,a) + \log C - \log p(a_t|s_t)$ and with a uniform reference policy. We can now assign familiar names to these messages, by defining $Q(s,a)=\log \beta(s,a)$ and $V(s)=\log \beta(s) - \log C$. Our message passing updates now resemble soft variants of Bellman backup equations:
\[
V(s_t) = \log \int_{\mathcal{A}} \exp\{Q(s_t,a_t)\} da_t
\]
\[
Q(s_t,a_t) = [r(s_t, a_t) +\log C] + \log E_{s_{t+1}}[\exp\{V(s_{t+1})\}]
\]
The constant $\log C$ term can be absorbed into the reward function to exactly match the equations we presented in Section~\ref{sec:graphical_model_control}, but we leave the term explicit for clarity of explanation. For the fixed horizon task we presented, adding a constant offset to the reward cannot change the optimal policy. As previously mentioned in Section~\ref{sec:variational_inf}, under deterministic dynamics, $Q(s_t,a_t) = r(s_t,a_t) + V(s_{t+1})$, which aligns with MaxCausalEnt~\citep{Ziebart10} and soft Q-learning~\citep{Haarnoja2017,Nachum17}.

From these value functions, we can easily obtain the optimal policy $p(a_t|s_t, e_{1:T}=1)$. First note that due to conditional independence, $p(a_t|s_t, e_{1:T}=1)=p(a_t|s_t, e_{t:T}=1)$. Applying Bayes' rule, we now have:
\[
p(a_t|s_t, e_{t:T}=1) = \frac{p(e_{t:T}=1|s_t,a_t)p(a_t|s_t)}{p(e_{t:T}=1|s_t)}
= \frac{\beta(s_t,a_t)C}{\beta(s_t)}
= \exp\{Q(s_t,a_t) - V(s_t)\}
\]

\section{Control as Variational Inference}
\label{sec:variational_inf}
Performing inference directly in the graphical model for control produces solutions that are optimistic with respect to stochastic dynamics, and produces risk-seeking behavior. This is because posterior inference is not constrained to force $p(s_{t+1}|s_t,a_t,e_{1:T}) = p(s_{t+1}|s_t,a_t)$: that is, it assumes that, like the action distribution, the next state distribution will ``conspire'' to make positive outcomes more likely. Prior work has sought to address this issue via the framework of causal entropy~\cite{Ziebart10}. To provide a more unified treatment of control as inference, we instead present a variational inference derivation that also addresses this problem. %
When conditioning the graphical model in Figure~\ref{fig:controlgraph} on $e_{1:T}=1$ as before, the optimal trajectory distribution is
\[
p(\tau|e_{1:T}) \propto p(s_1)\prod_{t=1}^{T-1} p(s_{t+1}|s_t,a_t)p(a_t|s_t)e^{r(s,a)}.
\]
We will assume that the action prior $p(a_t|s_t)$ is uniform without loss of generality, since non-uniform distributions can be absorbed into the reward term $e^{r(s,a)}$, as discussed in Appendix~\ref{app:event_derivations_mprl}.

The correct maximum entropy reinforcement learning objective emerges when performing variational inference in this model, with a variational distribution of the form \mbox{$q_\theta(\tau) = p(s_1)\prod_{t=1}^{T-1} p(s_{t+1}|s_t,a_t)q_\theta(a_t|s_t)$}. In this distribution, the initial state distribution and dynamics are forced to be equal to the true dynamics, and only the action conditional $q_\theta(a_t|s_t)$, which corresponds to the policy, is allowed to vary. Writing out the variational objective and simplifying, we get
\begin{align*}
-D_{KL}(q_\theta(\tau) || &p(\tau|e_{1:T}))
=
E_{\tau \sim q(\tau)}\left[\sum_{t=1}^T r(s_t,a_t) - \log q_\theta(a|s) \right].
\end{align*}
We see that we obtain the same problem as (undiscounted) entropy-regularized reinforcement learning, where $q_\theta(a|s)$ serves as the policy. For more in-depth discussion, see Appendix~\ref{app:event_derivations_causal_all}. We can recover the discounted objective by modifying the dynamics such that the agent has a $(1-\gamma)$ probability of transitioning into an absorbing state with 0 reward. 

We have thus derived how maximum entropy reinforcement learning can be recovered by applying variational inference with a specific choice of variational distribution to the graphical model for control. %

\section{Derivations for Event-based Message Passing Updates}
\label{app:event_derivations_mp}

\subsection{ALL query}
\label{app:event_derivations_mp_all}
The goal of the ALL query is to trigger an event at every timestep. Mathematically, we want trajectories such that $e_{1:T}=1$. As the ALL query is mathematically identical to MaxEnt RL, we redirect the reader to Appendix~\ref{app:event_derivations_mprl} for the derivation.

\subsection{ANY query}
\label{app:event_derivations_mp_any}

The goal of the ANY query is to trigger an event at least once. Mathematically, we want trajectories such that $e_1=1 \Or\ e_2=1\ ...\ e_T=1$.

First, we introduce a more concise notation by introducing a stopping time $t^* = \textrm{argmin}_{t\ge 0}\{e_t=1\}$ which denotes the first time that an event happens. Asking for the stopping time to be within a certain interval is the same as asking the event to happen at least once within that interval:
\[p(t^* \in [t,T]) = p(e_t=1 \Or\ e_{t+1}=1\ ...\ e_T=1)\]

We can now derive the message passing updates. We derive the state messages as:
\[\beta(s_t) = p(t^* \in [t,T]|s_t) =  \int_{\mathcal{A}} p(t^* \in [t,T]|s_t,a_t)p(a_t|s_t)da_t \]

The state-action message can be derived as:
\begin{align*}
\beta(s_t, a_t) &= p(t^* \in [t,T]|s_t, a_t)\\
&= p(e_t=1|s_t, a_t) + p(t^* \in [t+1,T]|s_t, a_t)
- p(e_t=1|s_t,a_t) p(t^* \in [t+1,T]|s_t, a_t) \\ 
&= p(e_t=1|s_t, a_t) + p(e_t=0|s_t, a_t) p(t^* \in [t+1,T]|s_t, a_t)  \\
&= p(e_t=1|s_t, a_t) + p(e_t=0|s_t, a_t) \int_{\mathcal{S},\mathcal{A}} p(t^* \in [t+1,T],s_{t+1},a_{t+1}|s_t,a_t)ds_{t+1}da_{t+1}  \\
&= p(e_t=1|s_t, a_t) + p(e_t=0|s_t, a_t) E_{s_{t+1}}\left[ \int_{\mathcal{A}} p(t^* \in [t+1,T]|s_{t+1},a_{t+1})p(a_{t+1}|s_{t+1})da_{t+1}\right]  \\
&= p(e_t=1|s_t, a_t) + p(e_t=0|s_t, a_t) E_{s_{t+1}}[\beta(s_{t+1})]  \\
\end{align*}

We can now define our Q and value functions as log-messages as done in Appendix~\ref{app:event_derivations_mprl} to obtain the following backup rules:

\[V(s_t) = \log \int_{\mathcal{A}} \exp\{Q(s_t,a_t)\} da_t\]
\[Q(s_t, a_t) = \log \left(p(e_t=1|s_t,a_t) + p(e_t=0|s_t,a_t)E_{s_{t+1}}[\exp\{V(s_t)\}]\right)\]

One caveat here is that the policy, $p(a_t|s_t, t^* \in [t,T])$, always seeks to make the event happening in the future, which we refer to as the \textit{seeking} policy. The correct non-seeking policy would be indifferent to actions after the event has happened. However, in terms of achieving the objective, both policies will behave exactly the same until the event is triggered, after which the behavior of the policy will no longer matter. For example, if we operate in the first exit scenario, and consider the episode terminated after the goal event is achieved, then we never encounter the scenario when the event occurs in the past. 

If we would like to compute the non-seeking policy, we can compute a forward pass which keeps track of the probability that the event has happened:
\[
p(t^* \in [1, t]|s_{1:t},a_{1:t}) = p(e_t=1|s_t,a_t) + p(e_t=0|s_t,a_t)p(t^* \in [1, t-1]|s_{1:t-1},a_{1:t-1})
\]
We can then use this forward message in conjunction with our backward messages to obtain a non-seeking policy as:
\[
p(a_t|s_{1:t}, a_{1:t-1}, t^* \in [1,T]) 
= \frac{p(a_t|s_{1:t}, a_{1:t}, t^* \in [1,T])p(a_t|s_t)}{p(a_t|s_{1:t}, a_{1:t-1}, t^* \in [1,T])}
\]

Where
\begin{align*}
p(t^* \in [1,T]|s_{1:t}, a_{1:t}) 
&= p(t^* \in [1,t-1]|s_{1:t-1}, a_{1:t-1}) + p(t^* \notin [1,t-1]|s_{1:t-1}, a_{1:t-1})p(t \in [t,T]|s_t, a_t)
\end{align*}
\begin{align*}
p(t^* \in [1,T]|s_{1:t}, a_{1:t-1}) 
&= \int_{\mathcal{A}} p(t^* \in [1,T]|s_{1:t}, a_{1:t})p(a_t|s_t) da_t
\end{align*}
Note that while the policy is conditioned on all past states and actions, it only depends on them through the forward message, or the cumulative probability that the event has happened.

\section{Derivations for Variational Objectives}
\label{app:event_derivations_causal}

\subsection{ALL query}
\label{app:event_derivations_causal_all}
We briefly reviewed the variational derivation for standard RL in Section~\ref{sec:variational_inf}. In this section, we present a more thorough derivation under the events framework and additionally discuss extensions to discounted formulations.

First, we write down the joint trajectory-event distribution, which is simply the product of all factors in the graphical model:
\[
p(\tau ,e_{1:T}=1) = p(s_1)\prod_{t=1}^{T-1} p(s_{t+1}|s_t, a_t) p(a_t|s_t) p(e_t=1|s_t,a_t)
\]

We can obtain the optimal trajectory distribution by conditioning and setting the reference policy $p(a_t|s_t)$ as the uniform distribution:
\[
p(\tau|e_{1:T}=1) \propto p(s_1)\prod_{t=1}^{T-1} p(s_{t+1}|s_t, a_t)p(e_t=1|s_t,a_t)
\]

We now perform variational inference with a distribution of the following form, where the dynamics have been forced to equal the true dynamics of the MDP: 
\[q_\theta(\tau) = p(s_1)\prod_{t=1}^{T-1} p(s_{t+1}|s_t,a_t)q_\theta(a_t|s_t)\]
Here, $q_\theta(a_t|s_t)$ is the only term that is allowed to vary, and represents the learned policy. When we minimize the KL divergence between $q$ and $p$, the dynamics terms cancel and we recover the following entropy-regularized policy objective:
\begin{align*}
-D_{KL}(q_\theta(\tau) || p(\tau|e_{1:T}=1))
&= -E_{q_\theta}[\sum_{t=0}^T \log q_\theta(a_t|s_t) - \sum_{t=0}^T \log p(e_t=1|s_t, a_t)] + C\\
&= E_{q_\theta}[\sum_{t=0}^T \log p(e_t=1|s_t, a_t) + H(\pi(\cdot|s_t))] + C 
\end{align*}
The constant C is due to proportionality in the optimal trajectory distribution, and can be ignored in the optimization process.

If we define the empirical returns $\hat{Q}$ as $\hat{Q}(s_t, a_t) = \sum_{t'=t}^T \log p(e_{t'}=1|s_{t'},a_{t'})$, we can write the returns recursively as:
\[\hat{Q}(s_t, a_t) = \log p(e_t=1|s_t, a_t) + \hat{Q}(s_{t+1}, a_{t+1})\]

In this discounted case, we consider the case when the dynamics has a $(1-\gamma)$ chance of transitioning into an absorbing state with reward or $\log p(e_t=1|s_t, a_t)=0$. This means we now adjust the recursion as:
\[\hat{Q}(s_t, a_t) = \log p(e_t=1|s_t, a_t) + \gamma \hat{Q}(s_{t+1}, a_{t+1})\]

\subsection{ANY query}
\label{app:event_derivations_causal_any}

As with our derivation in the RL case, we begin by writing down our trajectory distribution. Our target trajectory distribution is $p(\tau | t^* \in [1,T])$, or trajectories where the event happens at least once.

First, we can use Bayes' rule to obtain:
\[
\log p(\tau |  t^* \in [1,T])
= \log p( t^* \in [1,T] | \tau) + \log p(\tau) - \log p(t^* \in [1,T])
\]

The last term is a proportionality constant with respect to the trajectories. The second term is the trajectory distribution induced by the reference policy. The first term can be simplified further.

Note that the probability that the event first happens at $t^*$ is $p(t^*=t|\tau) = p(e_t=1|s_t, a_t)\prod_{t'=1}^{t-1}p(e_{t'}=0|s_{t'}, a_{t'})$ (i.e. the event happens at $t^*$ but not before). Now we can write:

\[
p( t^* \in[1,T]|\tau) 
= \sum_{t=1}^T p(t^*=t|\tau)
= \sum_{t=1}^T p(e_t=1|s_t, a_t)\prod_{t'=1}^{t-1}p(e_{t'}=0|s_{t'}, a_{t'})
\]

To write down a recursion, we now define the quantity $\hat{\beta}(s_{t:T}, a_{t:T}) = p(t^*\in[t,T]| s_{t:T}, a_{t:T})$. We can now express the above term recursively as:
\begin{align*}
\sum_{t=1}^T p(e_t=1|s_t, a_t)\prod_{t'=1}^{t-1}p(e_{t'}=0|s_{t'}, a_{t'})
&= p(e_1=1|s_1, a_1) + p(e_1=0|s_1,a_1)\hat{\beta}(s_{2:T}, a_{2:T}) \\
&= \hat{\beta}(s_{1:T}, a_{1:T})
\end{align*}
Thus, if we define our empirical Q-function $\hat{Q}(s_t, a_t) = \log \hat{\beta}(s_{1:T}, a_{1:T})$, our recursion now becomes:
\[
\hat{Q}(s_t, a_t) 
= \log( p(e_t=1|s_t, a_t) + p(e_t=0|s_t, a_t) e^{\hat{Q}(s_{t+1}, a_{t+1})})
\]

Using the same variational distribution $q_\theta(\tau) = p(s_1)\prod_{t=1}^T p(s_{t+1}|s_t, a_t) q_\theta(a_t|s_t)$ as before, we can write our optimization objective as:
\begin{align*}
-D_{KL}(q_\theta(\tau) || p(\tau|t^* \in [1,T])) = E_q[ \hat{Q}(s_{1:T}, a_{1:T}) - \sum_{t=1}^T \log q_\theta(a_t|s_t)] + C
\end{align*}
Where the constant $C$ absorbs terms from the reference policy $p(a_t|s_t)$ which we set to uniform, and the proportionality constant $\log p(t^* \in [1,T])$.

To achieve a discounted objective case, we consider the case when the dynamics has a $(1-\gamma)$ chance of transitioning into an absorbing state where the event can never happen $p(e_t=1|s_t,a_t)=0$. Note that this is different from the all query. This means we now adjust the recursion as:
\[
\hat{Q}(s_t, a_t) 
= \gamma \log\left( p(e_t=1|s_t, a_t) + p(e_t=0|s_t, a_t) e^{\hat{Q}(s_{t+1}, a_{t+1})}\right) + (1-\gamma) \log p(e_t=1|s_t, a_t)
\]

\section{Policy Gradients for Events}
\label{app:policy_gradient}

Because the ALL query is mathematically identical to standard RL, we do not derive the policy gradient estimator here.

For the ANY query, we consider the objective \[J(\pi) = E_\pi[ \hat{Q}(s_{1:T}, a_{1:T}) - \sum_{t=1}^T \log \pi(a_t|s_t)]\]. For simplicity we disregard the entropy term as that portion remains unchanged from standard RL.

Applying logarithmic differentiation, and simplifying, we can obtain the gradient estimator.
\[
E_\pi[\sum_{t=1}^T \nabla \log \pi(a_t|s_t) (\hat{Q}(s_{1:T}, a_{1:T}) - \log \pi(a_t|s_t))]
\]

The next step is that we wish to only consider future returns, i.e. we wish to replace $\hat{Q}(s_{1:T}, a_{1:T})$ with $\hat{Q}(s_{t:T}, a_{t:T})$. First, note that before the event happens before $t$, then $\hat{Q}(s_{1:T}, a_{1:T})$ and $\hat{Q}(s_{t:T}, a_{t:T})$ are identical, but if $t$ is after then event then the returns estimator should be 0. Thus, we need to keep track of the cumulative probability that an event occurs and rewrite the estimator as:
\[
E_\pi[\sum_{t=1}^T \nabla \log \pi(a_t|s_t) p(t \le t^*|s_{1:t}, a_{1:t}) (\hat{Q}(s_{t:T}, a_{t:T}) - \log \pi(a_t|s_t))]
\]

%To obtain an estimator that only depends on future returns, we assume that the event never happens in the past and set $p(t \le t^*|s_{1:t}, a_{1:t})=1$. The justification is that after the event happens, we are indifferent to the behavior of the policy, and this estimator encourages the policy to continue to trigger the event. We also discuss this point towards the end of Appendix~\ref{app:event_derivations_mp_any} - it is exactly correct in a first-exit scenario where the episode terminates upon triggering the event so that the case when the event happened in the past is impossible, and otherwise still provides reasonable behavior.

\section{Variational Inverse Control with Events (VICE)}
\label{app:irl}
In this section, we explicitly write down the MLE objective for the inverse formulation of each query type (AT, ALL, ANY).

We then show that we can train a sampler for the model by optimizing a trajectory-based objective corresponding to the inference procedures outlined in Appendix~\ref{app:event_derivations_causal}. The statement we show for each query type is that the KL between trajectory distributions upper bounds the KL between our sampler and the model we wish to draw samples from.

\subsection{AT query VICE}

In the AT query, we assume we observe states and actions where the event occurred at a specific timestep, denoted as $t$. We assume our data comes from the distribution $p_{data}(s_t,a_t|e_t=1)$ 

The maximum likelihood objective is:
\[ \mathcal{L}_{AT}(\theta) = -E_{p_{data}}\left[\log p_\theta(s_t, a_t|e_t=1)\right]\]

We now derive the objective for training our sampler $q(s_t, a_t)$ so that it matches $p_\theta(s_t, a_t)$. By the chain rule for KL divergence, we have the upper-bound $D_{KL}(q(s_t,a_t) || p_\theta(s_t, a_t|e_t=1)) \le D_{KL}(q(\tau)||p_\theta(\tau|e_t=1))$. After obtaining $q(\tau)$, we can sample states and actions by executing full trajectories and picking the states and actions that correspond to timestep $t$.

\subsection{ALL query VICE}

In the ALL query, we assume our data comes from the average distribution of states and actions along trajectories where the event happens at all timesteps (averaged over timesteps) $p_{data}(s,a|e_{1:T}=1) = \frac{1}{T}\sum_{t=1}^T p_{data}(s_t, a_t | e_{1:T}=1)$.
This is similar to matching the occupancy measure of a policy, which is equivalent to inverse reinforcement learning as shown by~\cite{Ho16b}. 

The maximum likelihood objective is:
\[ \mathcal{L}_{ALL}(\theta) = -E_{p_{data}}\left[\log p_\theta(s_t, a_t|e_{1:T}=1)\right]\]

We can upper-bound the KL-divergence of interest between the sampler and the model with a KL-divergence on trajectories as:
\begin{align*}
 &D( \frac{1}{T}\sum_t q(s_t, a_t)||\frac{1}{T}\sum_t p_\theta(s_t, a_t|e_{1:T}=1))\\
&\le \frac{1}{T}\sum_t  D( q(s_t, a_t)||p_\theta(s_t, a_t|e_{1:T}=1))\\
&\le \frac{1}{T}\sum_t  D( q(\tau)||p_\theta(\tau|e_{1:T}=1))\\
 &= D(q(\tau)||p_\theta(\tau|e_{1:T}=1))
\end{align*}
The first inequality comes from the log-sum inequality, and the second inequality comes from the chain rule for KL divergence.

\subsection{ANY query VICE}
In the ANY query formulation, we assume our data comes from the distribution of states and actions at the first timestep an event happens, $p_{data}(s_{t^*},a_{t^*}|t^* \in [1,T])$.

%We have substituted $p_{data}(t^* \in [1,T]|s, a) = \sum_t p_{data}(t^*=t|S_t=s, A_t=a)$ because each $t^*=t$ is a disjoint event. We can further simplify this as $\sum_t p_{data}(t^*=t|S_t=s, A_t=a) = p_{data}(e_1=1|S_1=s, A_1=s) + p_{data}(e_1=0|S_1=s, A_1=s)[p_{data}(e_2=1|S_2=s, A_2=s) + ...] = Q_{data}(s,a)$. Thus, assuming our model $p_\theta$ takes the same form, we obtain the maximum likelihood objective:

\[ \mathcal{L}_{ANY}(\theta) = -E_{p_{data}}\left[\log p_\theta(p_{data}(s_t,a_t|t^*=t))\right]\]
%and the corresponding gradient
%\[ \mathcal{L}(\theta) = -E_{p_{data}}\left[\log p_\theta(s,a|t^* \in [1,T])\right] - E_{p_\theta}[\log p_\theta(s,a|t^* \in [1,T])] \]
%Where $p_\theta(s,a|t^* \in [1,T]) = p_\theta(e_1=1|S_1=s, A_1=s) + p_\theta(e_1=0|S_1=s, A_1=s)[p_\theta(e_2=1|S_2=s, A_2=s) + ...] $.

To show that optimizing the trajectory distribution bounds, we first rewrite $p(s_{t^*},a_{t^*}|t^* \in [1,T])$ over timesteps as $p(s_{t^*},a_{t^*}|t^* \in [1,T]) = \sum_{t=1}^T p(s_t, a_t | t^*=t) p(t^*=t)$. 
\begin{lemma}
\label{joint_kl}

Let $\bX = (x_1, x_2, ...)$, $\bY= (y_1, y_2, ...)$. 
Let $\bar{\mu}$ denote a set of weights which sum to one, and denote $\bar{p}(\bX) =E_{\bar{\mu}}[p_i(x_i)]$,
and $\bar{p}(\bX, \bY) = E_{\bar{\mu}}[p_i(x_i, y_i)]$ denote convex combinations of the individual distributions $p_i$. Then,

\[
D(\bar{p}(\bX, \bY) ||\bar{q}(\bX, \bY)) \ge D(\bar{p}(\bX) ||\bar{q}(\bX))
\]

\end{lemma}
\begin{proof}
This statement directly follows from the chain rule for KL divergences, which implies:
\[
D(\bar{p}(\bX, \bY) ||\bar{q}(\bX, \bY))
= D(\bar{p}(\bX) ||\bar{q}(\bX)) + D(\bar{p}(\bX|\bY) ||\bar{q}(\bX|\bY))
\ge D(\bar{p}(\bX) ||\bar{q}(\bX))
\]

\end{proof}

Now, we can apply Lemma~\ref{joint_kl} to derive the upper-bound:
\begin{align*}
&D( \sum_t q(s_t, a_t)p(t^*=t)||\sum_t p_\theta(s_t, a_t|t^*=t)p(t^*=t))\\
&\le D( \sum_t q(\tau)p(t^*=t)||\sum_t p_\theta(\tau|t^*=t)p(t^*=t))\\
&= D( q(\tau)||p_\theta(\tau|t^* \in [1,T]))
\end{align*}
We can obtain samples from $q(s_{t^*}, a_{t^*})$ by executing full trajectories and using the first state when an event is triggered.

\subsection{Justification for using the discriminator}
In the previous section, we have justified the algorithm which updates the model via the gradient Eqn.~\ref{eqn:irl_grad}, by training a sampling policy that minimizes KL to the model distribution.

In Section~\ref{sec:airl}, we propose to implement the update via training a discriminator instead of an energy-based model $p_\theta(s,a|e=1)$. An approximate connection can be made in this case, which ignores changes in the state-distribution of the sampling policy. To see this, we represent the state-action marginal of the policy as $q(s,a) = q(a|s)\bar{p}(s)$, where $\bar{p(s)}$ is the state-marginal of the reference policy (set to uniform, see Appendix~\ref{app:event_derivations_mprl}). Note that this is not the state distribution induced by the policy, $q(s)$.

We can use Bayes rule to write our model as $p_\theta(s,a|e=1) \propto p_\theta(e=1|s,a)\bar{p}(a|s)\bar{p}(s)$, meaning our model is only parameterized by the event probability.

Following previous work~\cite{Finn16b}, we model the discriminator as $D_\theta(s,a) = \frac{p_\theta(s,a|e=1)}{p_\theta(s,a|e=1) + q(s,a)} = \frac{p_\theta(e=1|s,a) + C_\theta}{p_\theta(e=1|s,a) + C_\theta + q(a|s)}$, where $C_\theta$ is a learnable constant that corresponds to proportionality factors.

The inconsistency with using $q(s,a) = q(a|s)\bar{p}(s)$ instead of $q(s,a) = q(a|s)q(s)$ arises in the policy optimization objective, which is minimizing the KL between the latter quantity an the model. This means that we do not draw unbiased negative examples for training the discriminator, which is also noted in~\citep{Fu18}.

\section{Experiments}
\label{app:experiments}

\subsection{Experimental details for prespecified events}

On the \textit{Lobber} task, we use a diagonal gaussian policy where the mean is parametrized by a 32x32 neural network. We use a TRPO batch size of 40000 and train for 1000 iterations.

On the \textit{HalfCheetah} task, we use a diagonal gaussian policy where the mean is parametrized by a 32x32 neural network. We use a TRPO batch size of 10000 and train for 1000 iterations.

\subsection{Experimental details for learning event probabilities}
We evaluate the performance of VICE in learning event probabilities on the \textit{Ant},\textit{Maze}, and \textit{Pusher} tasks, providing comparisons to classifier-based methods. Although the binary indicator baseline is not comparable to VICE (since it observes the event while the other methods do not), we present comparisons to provide a general idea of the difficulty of the task. All experiments are run with five random seeds, and mean results are presented.

We use Gaussian policies, where the mean is parametrized by a neural network, and the covariance a learned diagonal matrix. The event distribution is represented by a neural network as well. Further hyperparameters are presented in Table \ref{tbl:hyperparameters}. 

On the \textit{Ant} task, both the policy mean network and event distribution network have two hidden layers with 200 units and ReLu activations. 

On the \textit{Maze} task, the mean network has two convolutional layers, with filter size $5 \times 5$, followed by two fully connected layers with $32$ units each with ReLu activations. The event distribution is represented using a convolutional neural network with two convolutional layers with $5 \times 5$ filters, and a final fully-connected layer with 16 units. 

On the \textit{Pusher} task, the policy is represented by a convolutional neural network with three convolutional layers, with a stride of 2 in the first layer, and a stride of 1 in the subsequent layers. We use a filter size of 3x3 in all the layers, and the number of filters are 64, 32 and 16. In line with prior work~\citep{finn16dsae}, we pre-train the convolutional layers using an auto-encoder loss on data collected from random policies. The fully connected part of the neural network consists of two layers, each having 200 units and ReLu activations to represent the policy. The event distribution is also represented by the same architecture.

\begin{table}[h]
\small
\begin{center}
\begin{tabular}{c|ccc}
\hline
 & Ant & Maze & Pusher \\
\hline
Batch Size & 10000 & 5000 & 10000\\
Iterations & 1000  & 150  & 1000 \\
Discount   & 0.99  & 0.99 & 0.99 \\
Entropy    & 0.1   & 0.1  & 0.01  \\
\# Demonstrations & 500 & 1000 & 10000\\
\hline 
\end{tabular}

\caption{\label{tbl:hyperparameters} \small 
Hyperparameters used for VICE on the Ant,Maze, and Pusher tasks}
\end{center}
\end{table}

\begin{figure}[h]
\centering
\begin{subfigure}[b]{0.66\textwidth}
  \includegraphics[width=\linewidth]{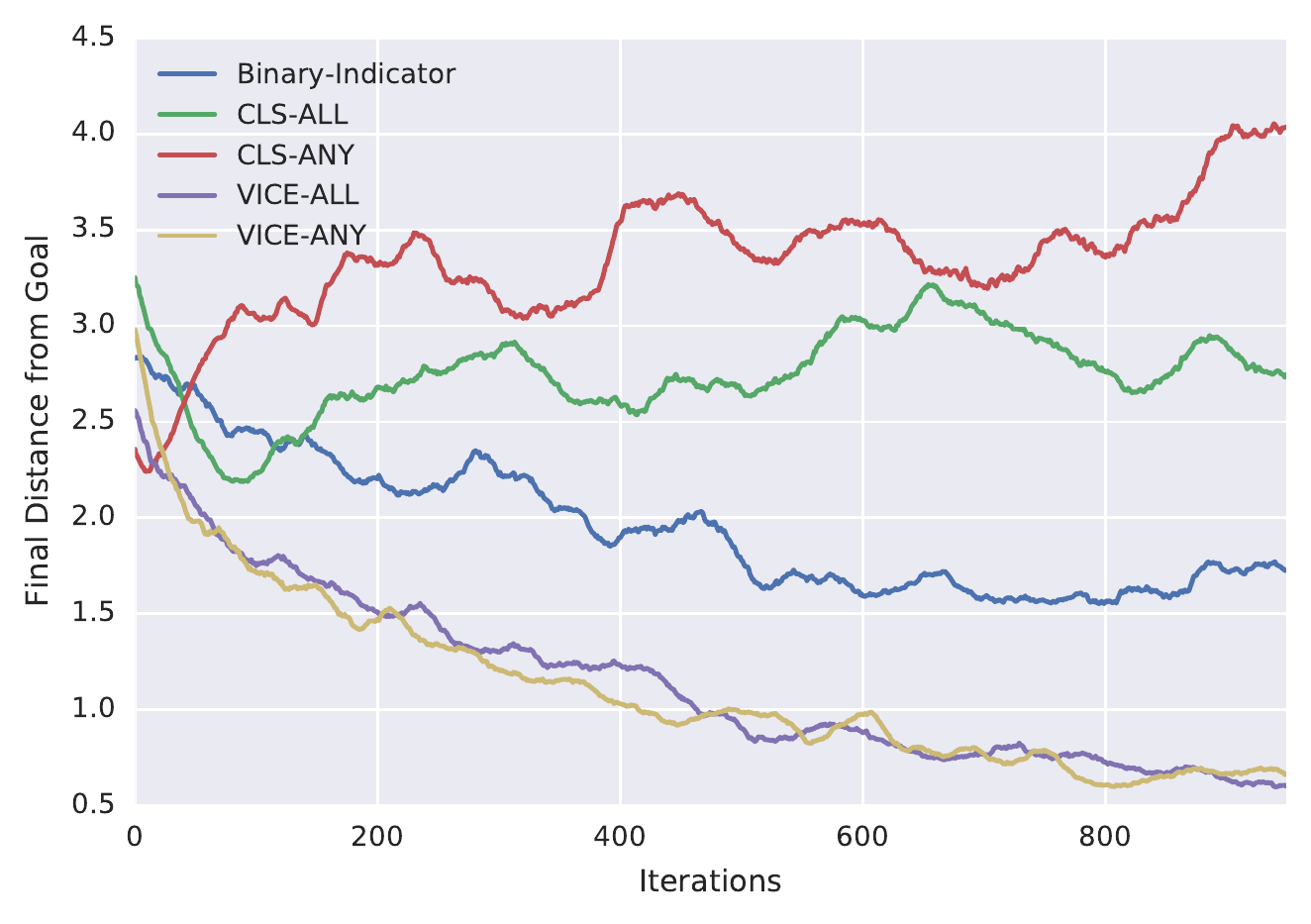}
   \caption{Ant}
   \label{fig:Ng1} 
\end{subfigure}
\begin{subfigure}[b]{0.66\textwidth}
  \includegraphics[width=\linewidth]{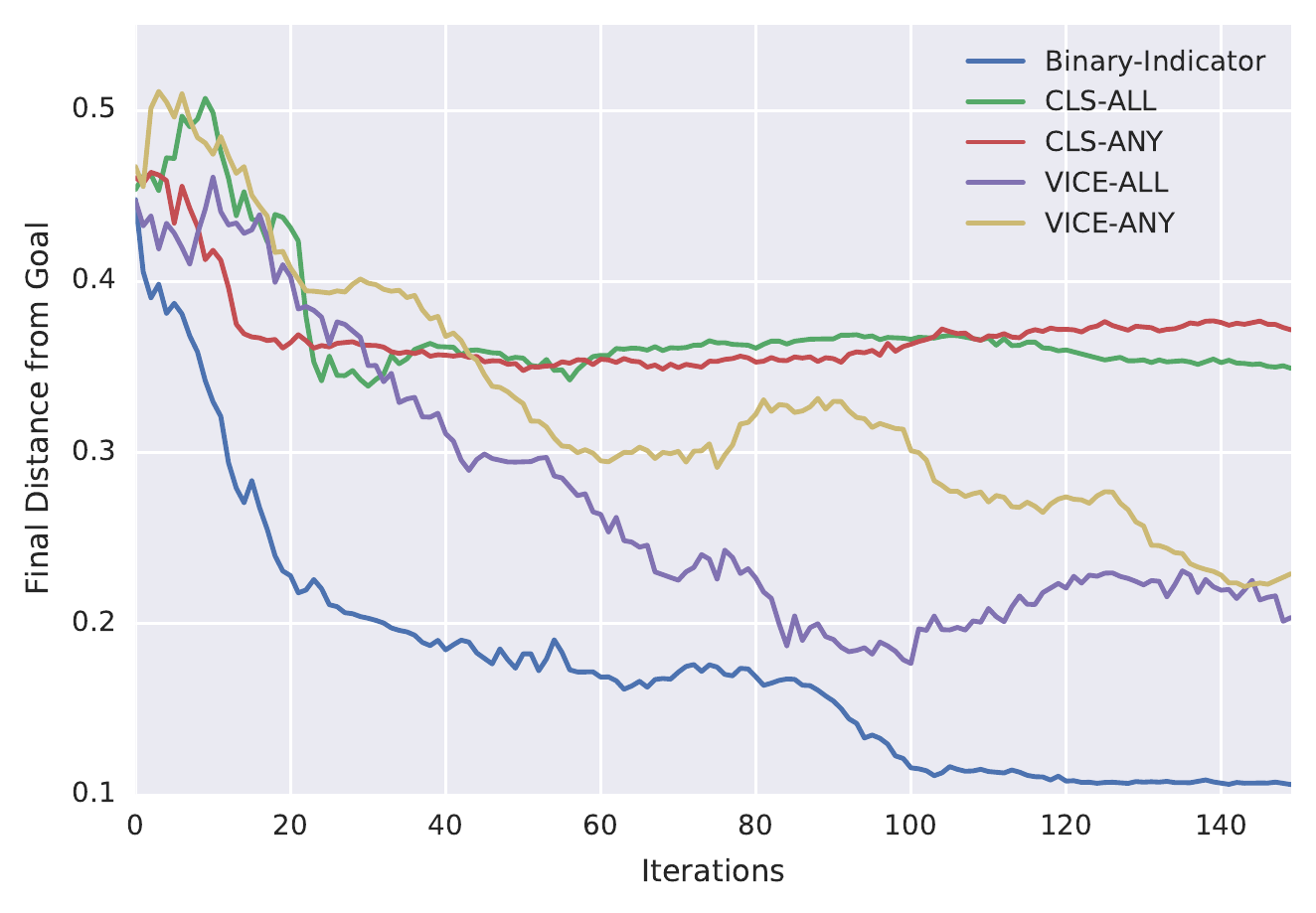}
   \caption{Maze}
   \label{fig:Ng2}
\end{subfigure}
\begin{subfigure}[b]{0.66\textwidth}
  \includegraphics[width=\linewidth]{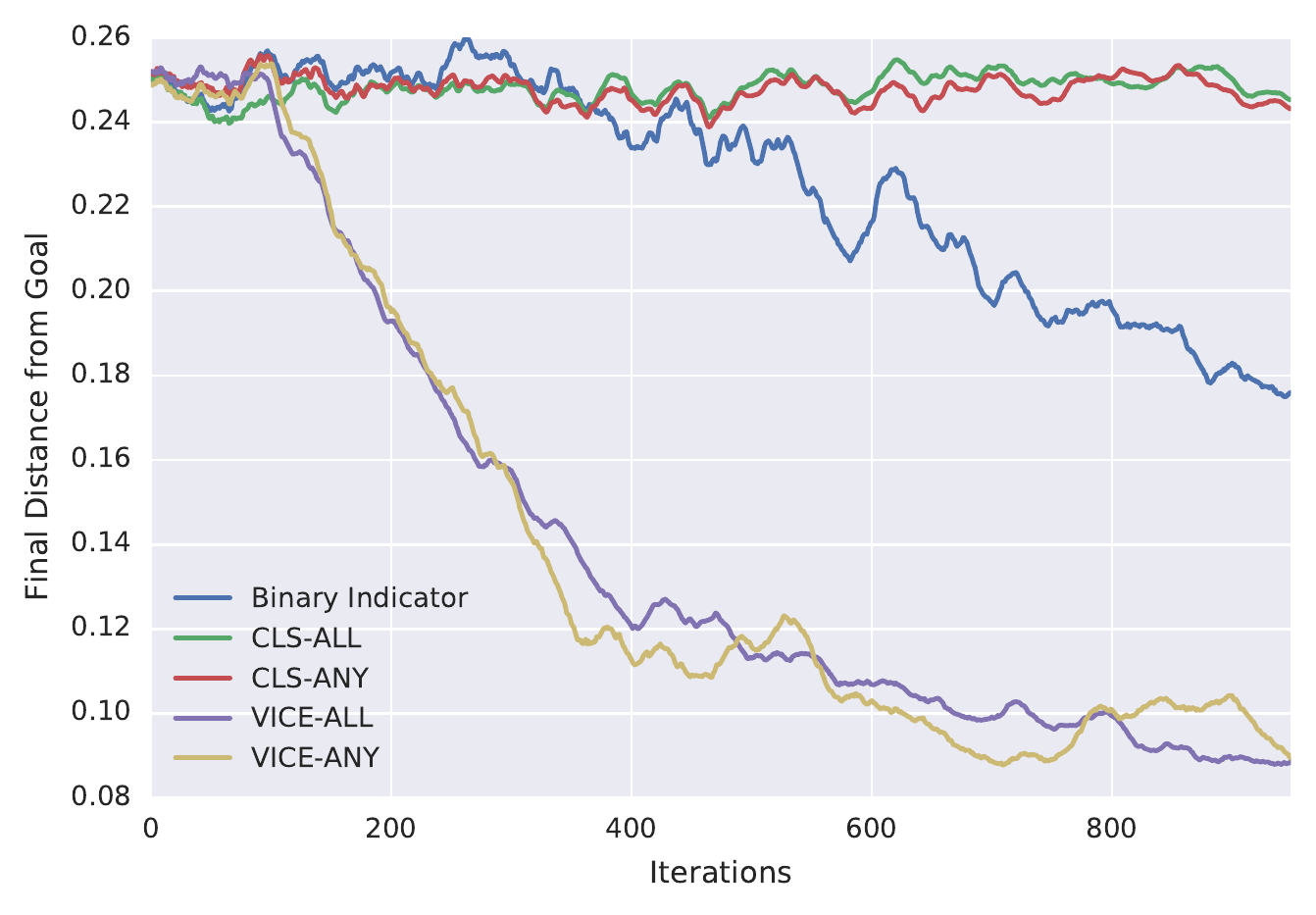}
   \caption{Pusher}
   \label{fig:Ng2}
\end{subfigure}

\caption{Learning curves for the various methods for the \textit{Ant},\textit{Maze}, and \textit{Pusher} tasks, averaged across five random seeds. On all three domains, VICE-ALL and VICE-ANY successfully solve the task consistently, while the naive classifier fails often. Although the binary indicator works reasonably on the Maze task, it fails to solve the task in the more challenging environments. }
\end{figure}

\clearpage
\subsection{Detailed learning curves for learning event probabilities}

\begin{figure}[h]
\centering
\begin{subfigure}[b]{0.3\textwidth}
\includegraphics[width=\textwidth]{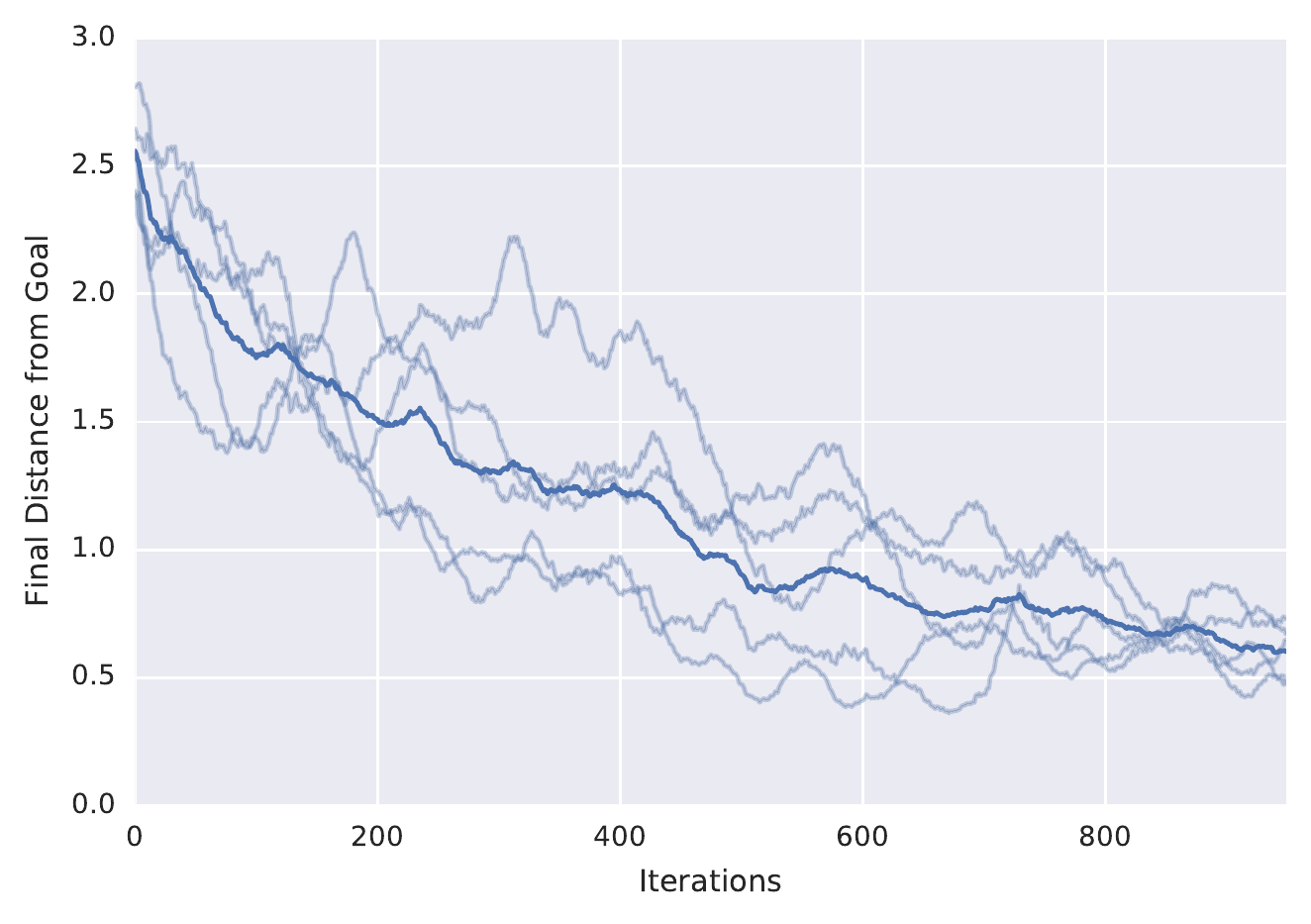}
\caption{Ant - VICE-ALL}
\end{subfigure}
\begin{subfigure}[b]{0.3\textwidth}
\includegraphics[width=\textwidth]{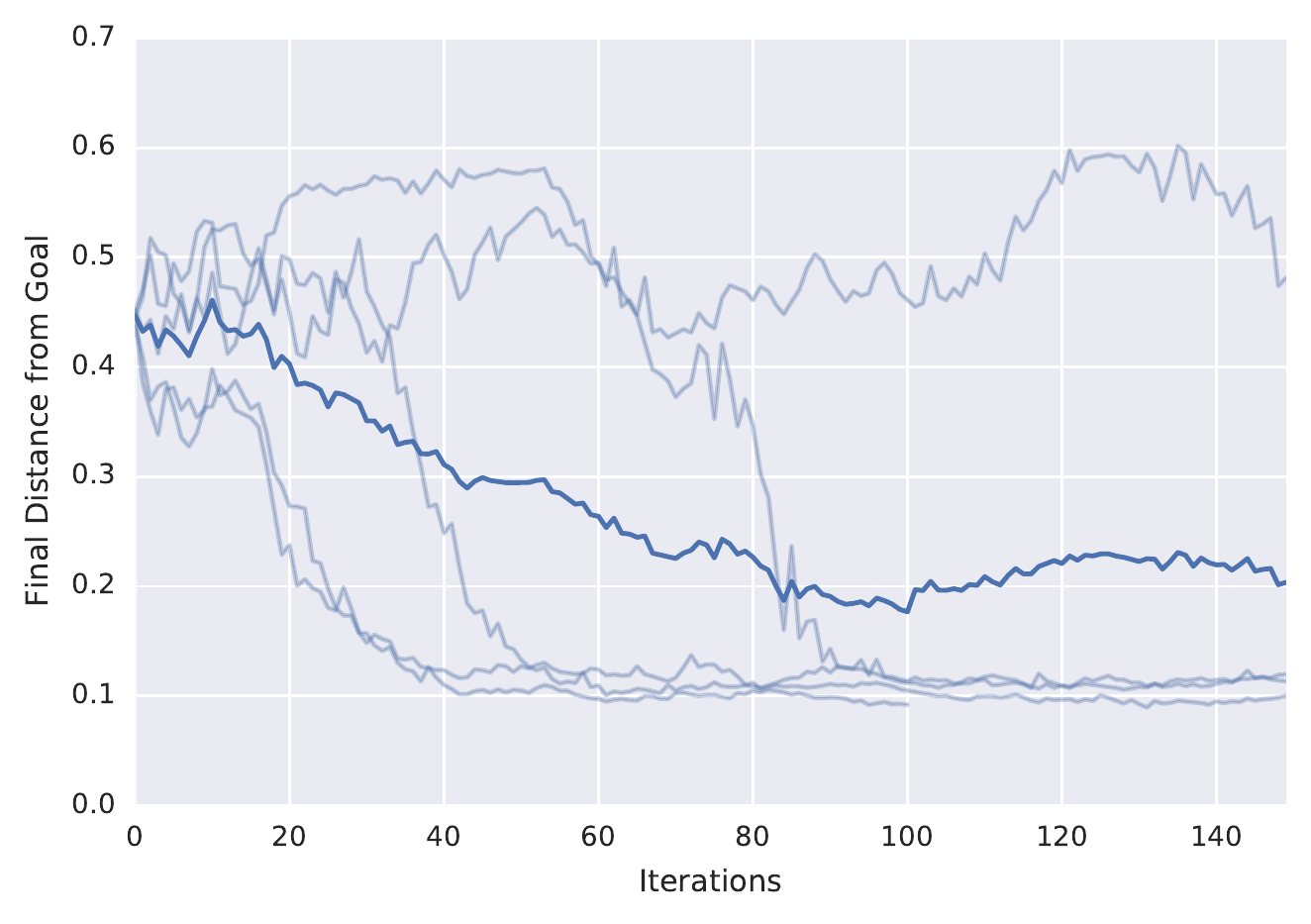}
\caption{Maze - VICE-ALL}
\end{subfigure}
\begin{subfigure}[b]{0.3\textwidth}
\includegraphics[width=\textwidth]{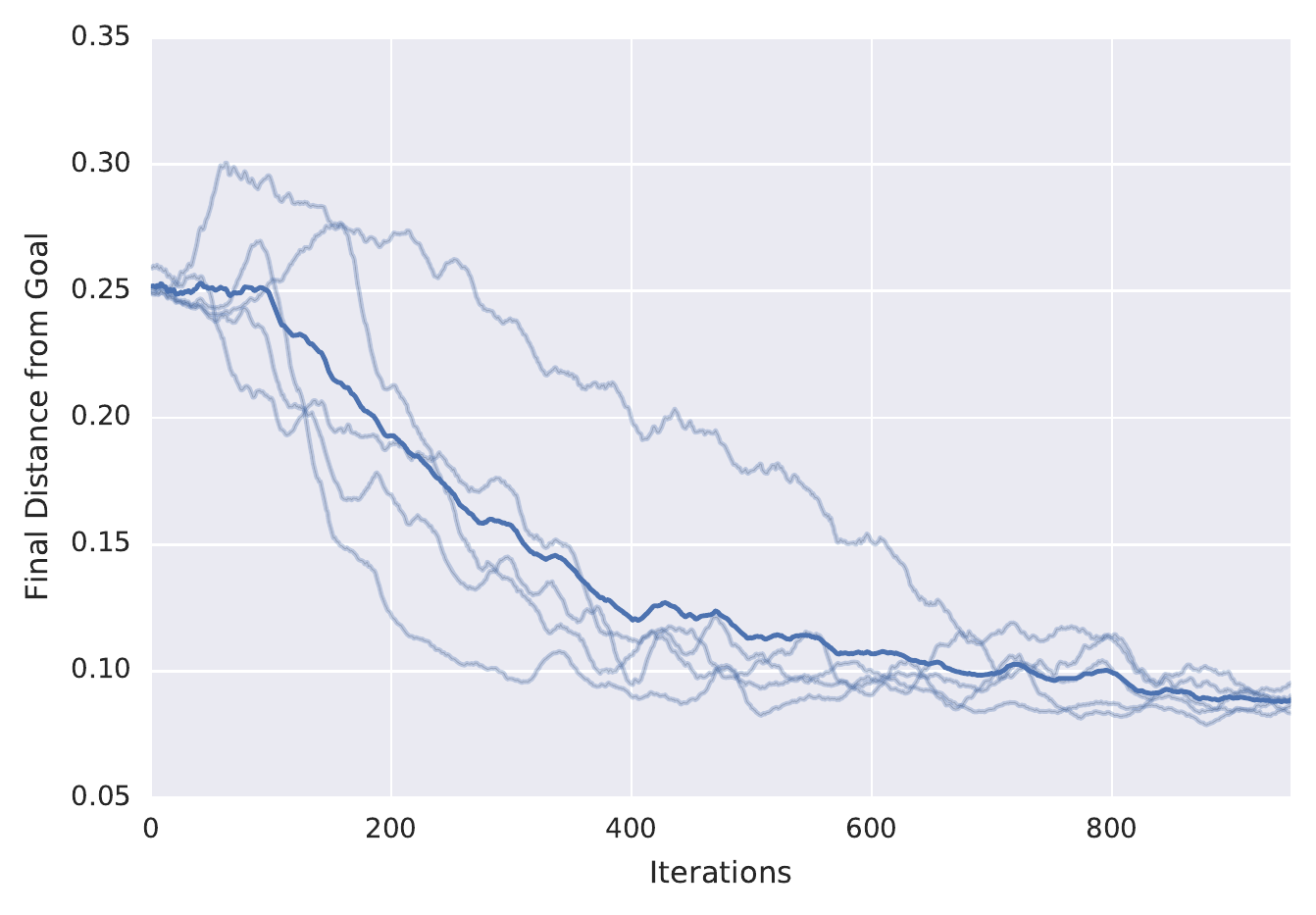}
\caption{Pusher - VICE-ALL}
\end{subfigure}\\
\begin{subfigure}[b]{0.3\textwidth}
\includegraphics[width=\textwidth]{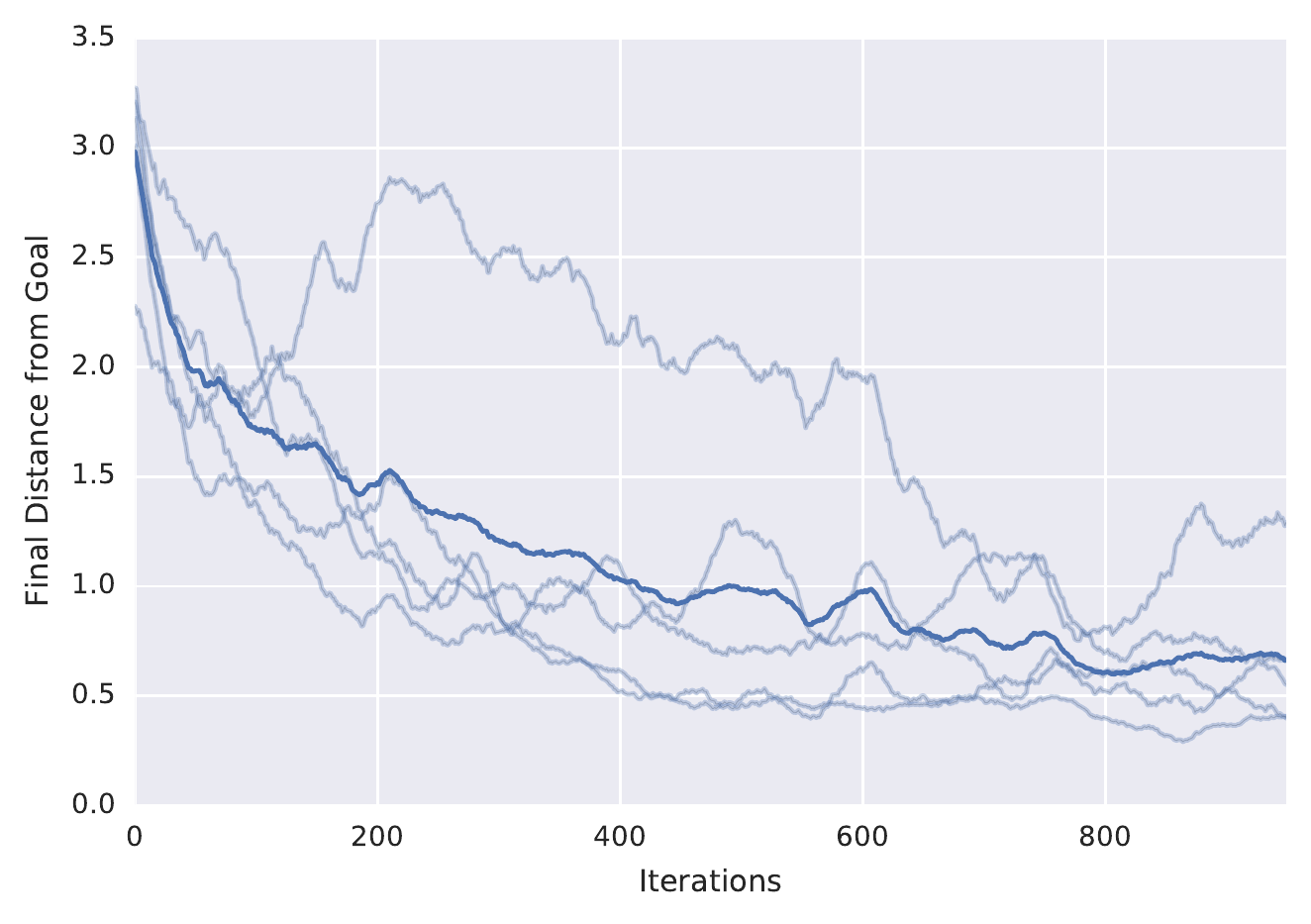}
\caption{Ant - VICE-ANY}
\end{subfigure}
\begin{subfigure}[b]{0.3\textwidth}
\includegraphics[width=\textwidth]{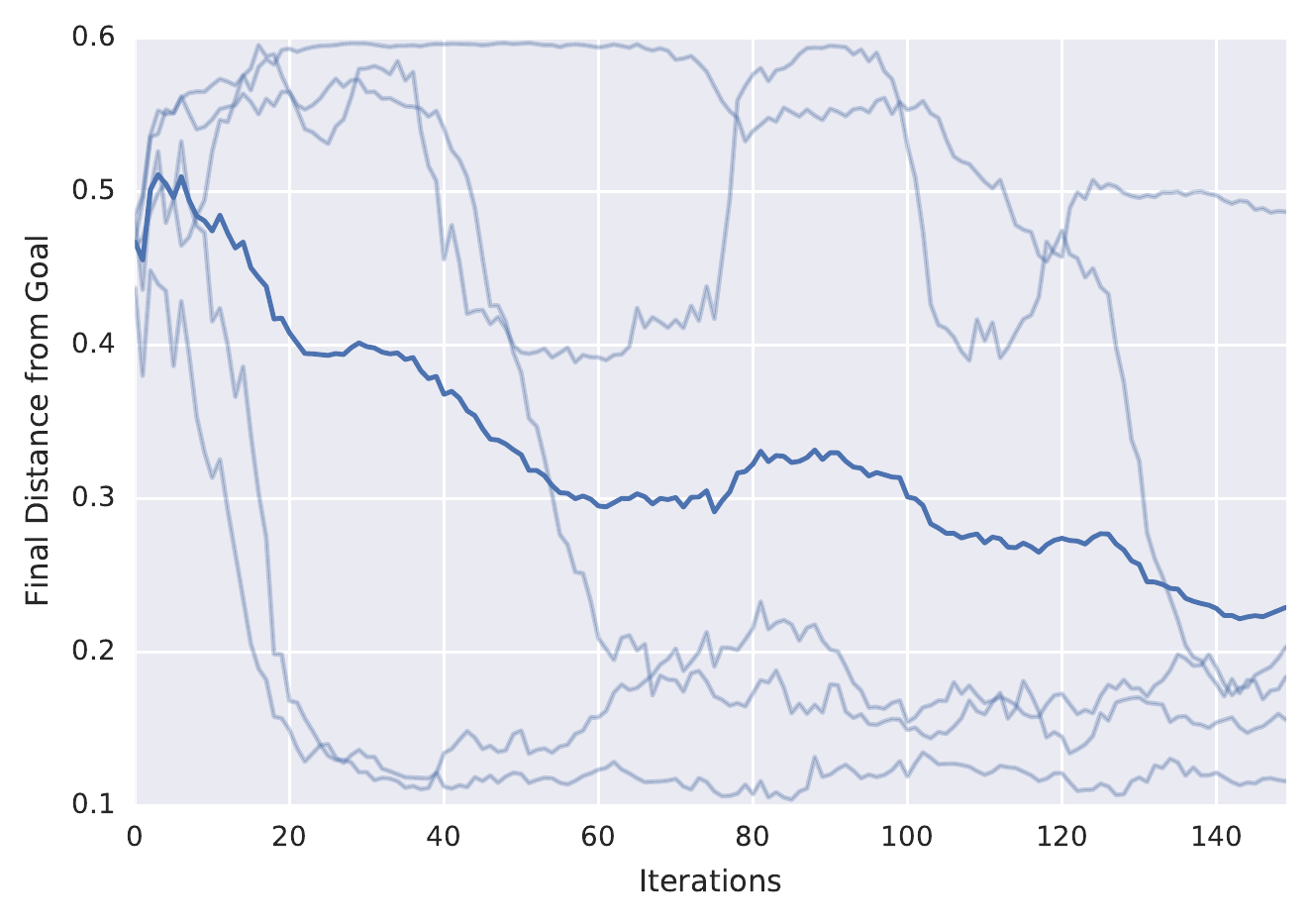}
\caption{Maze - VICE-ANY}
\end{subfigure}
\begin{subfigure}[b]{0.3\textwidth}
\includegraphics[width=\textwidth]{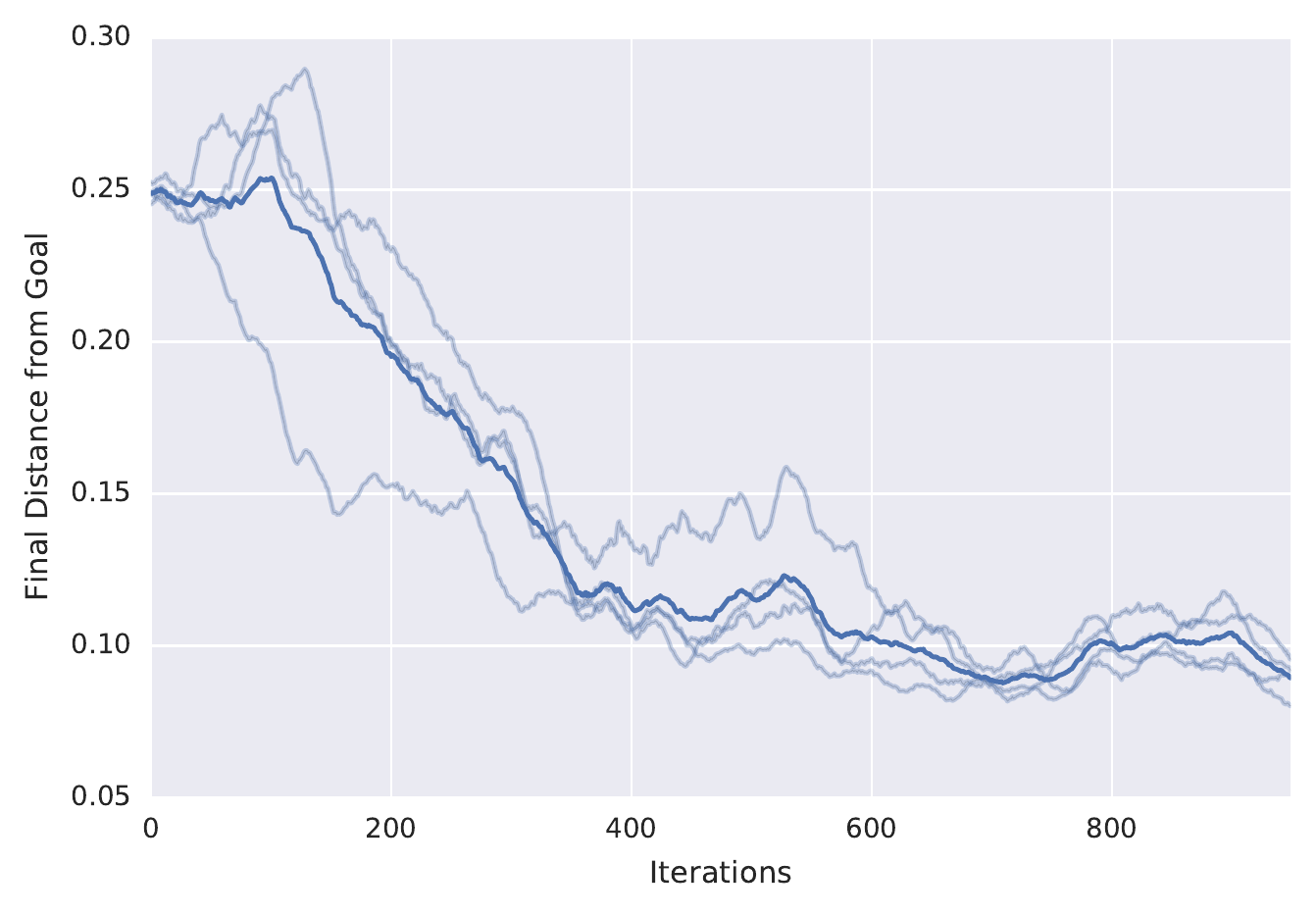}
\caption{Pusher - VICE-ANY}
\end{subfigure}\\
\begin{subfigure}[b]{0.3\textwidth}
\includegraphics[width=\textwidth]{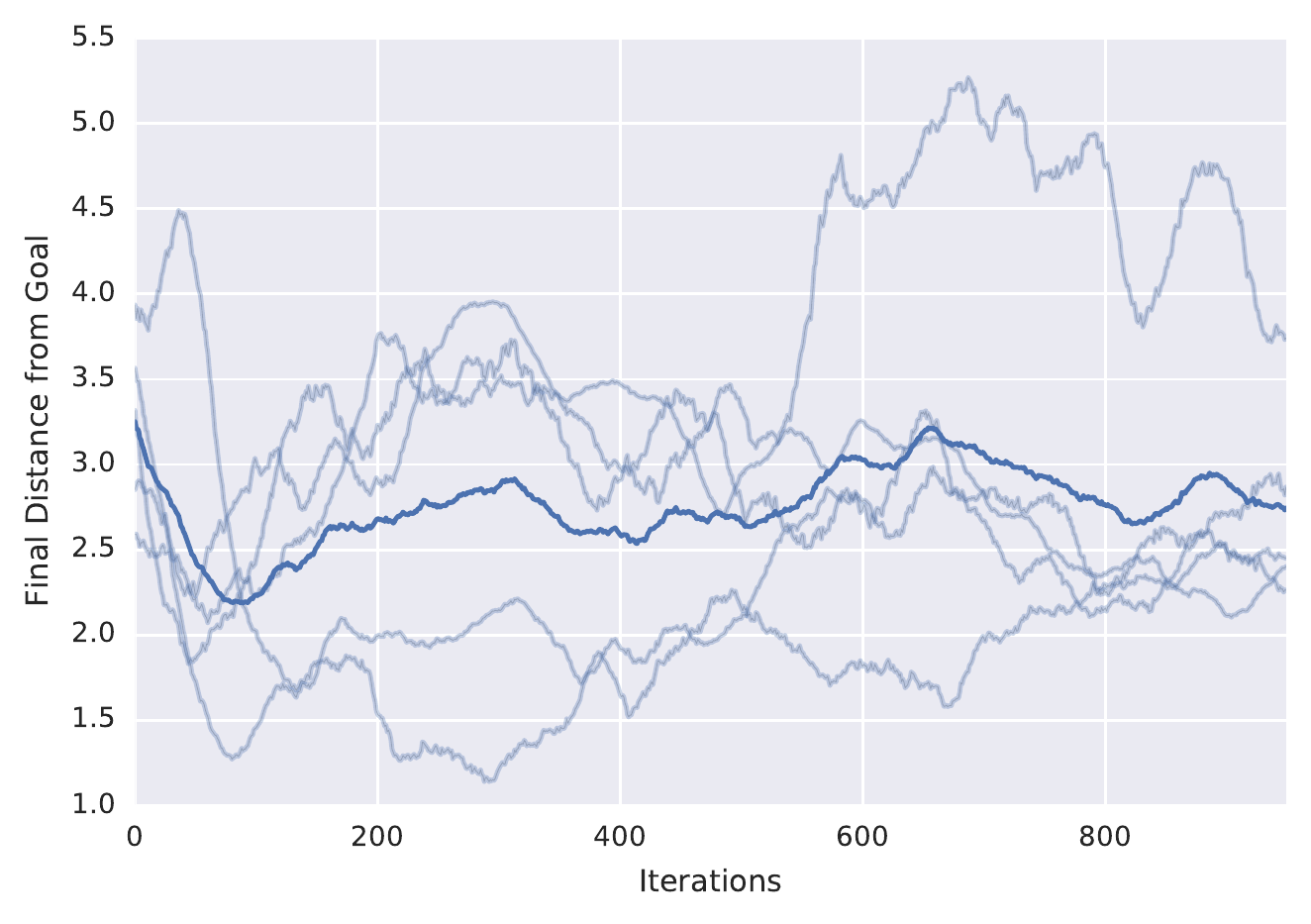}
\caption{Ant - CLS-ALL}
\end{subfigure}
\begin{subfigure}[b]{0.3\textwidth}
\includegraphics[width=\textwidth]{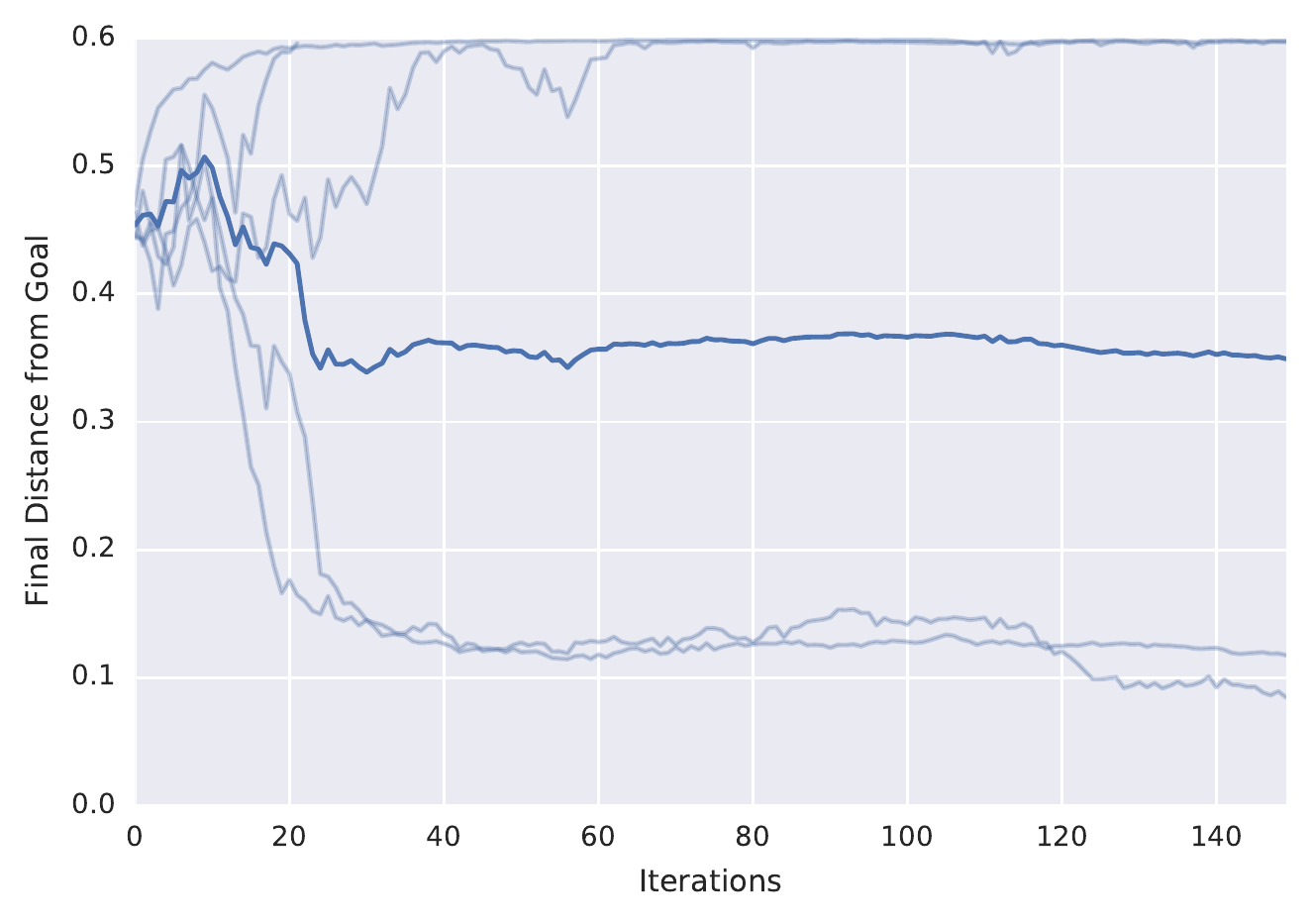}
\caption{Maze - CLS-ALL}
\end{subfigure}
\begin{subfigure}[b]{0.3\textwidth}
\includegraphics[width=\textwidth]{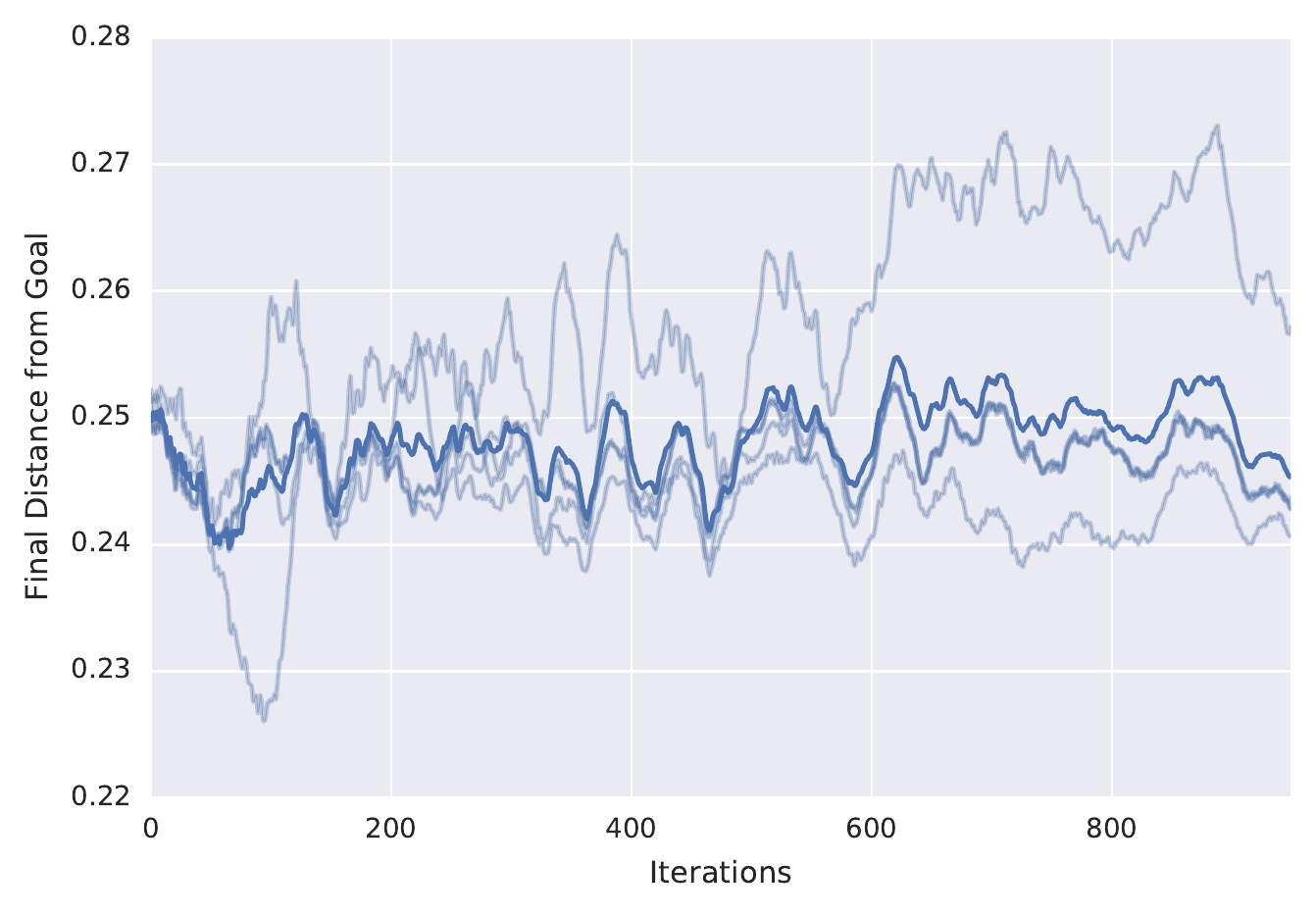}
\caption{Pusher - CLS-ALL}
\end{subfigure}\\
\begin{subfigure}[b]{0.3\textwidth}
\includegraphics[width=\textwidth]{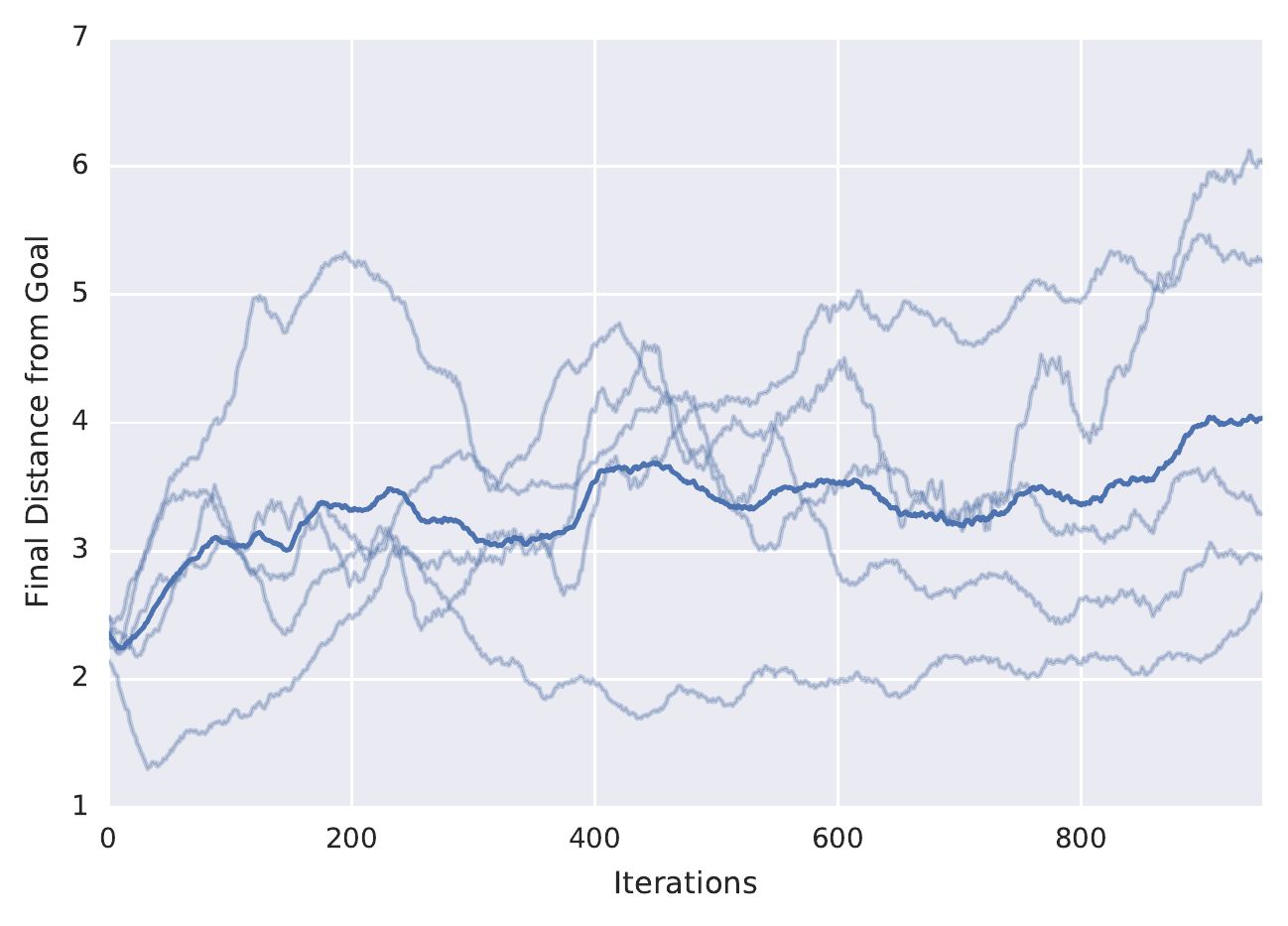}
\caption{Ant - CLS-ANY}
\end{subfigure}
\begin{subfigure}[b]{0.3\textwidth}
\includegraphics[width=\textwidth]{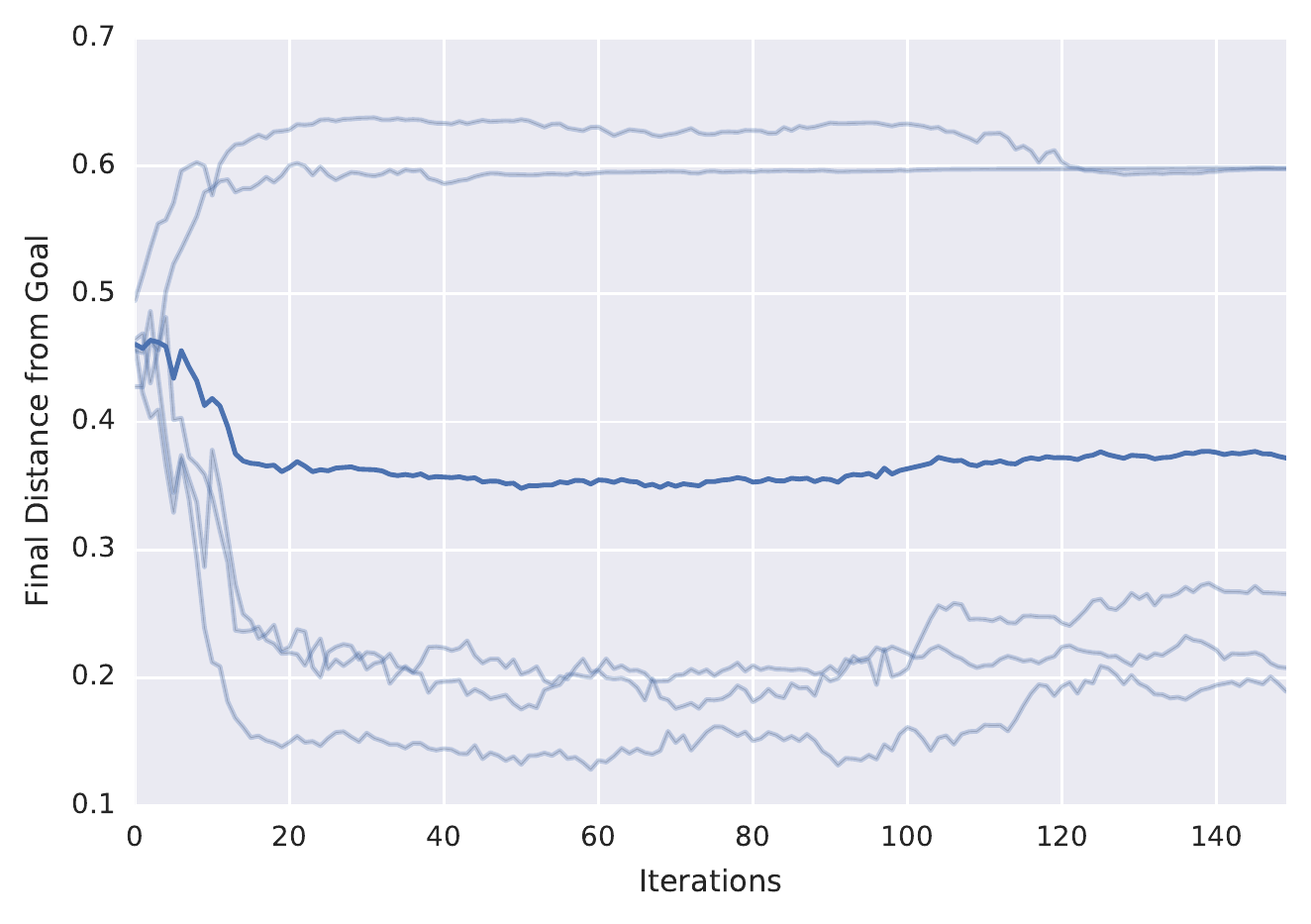}
\caption{Maze - CLS-ANY}
\end{subfigure}
\begin{subfigure}[b]{0.3\textwidth}
\includegraphics[width=\textwidth]{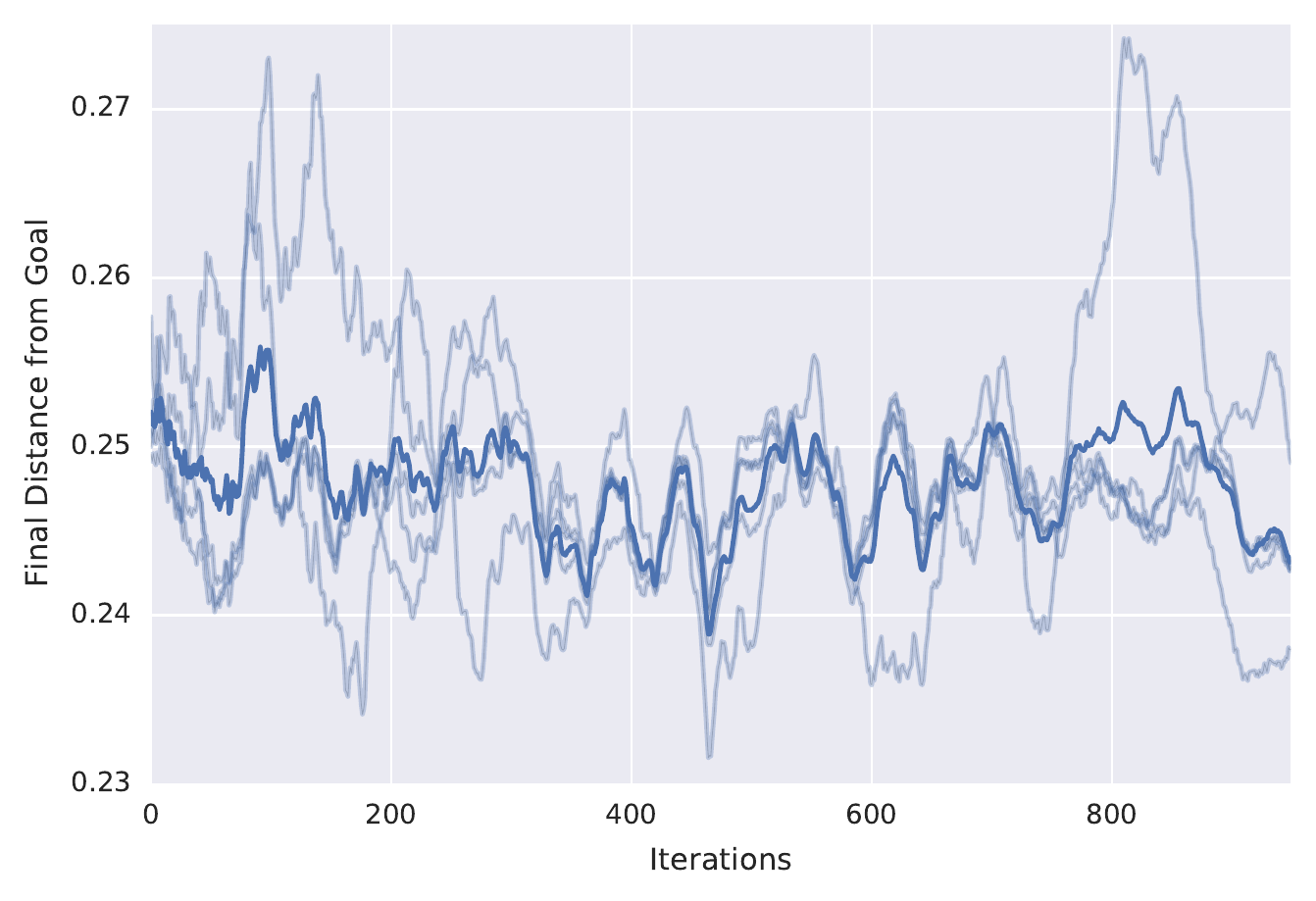}
\caption{Pusher - CLS-ANY}
\end{subfigure}\\
\begin{subfigure}[b]{0.3\textwidth}
\includegraphics[width=\textwidth]{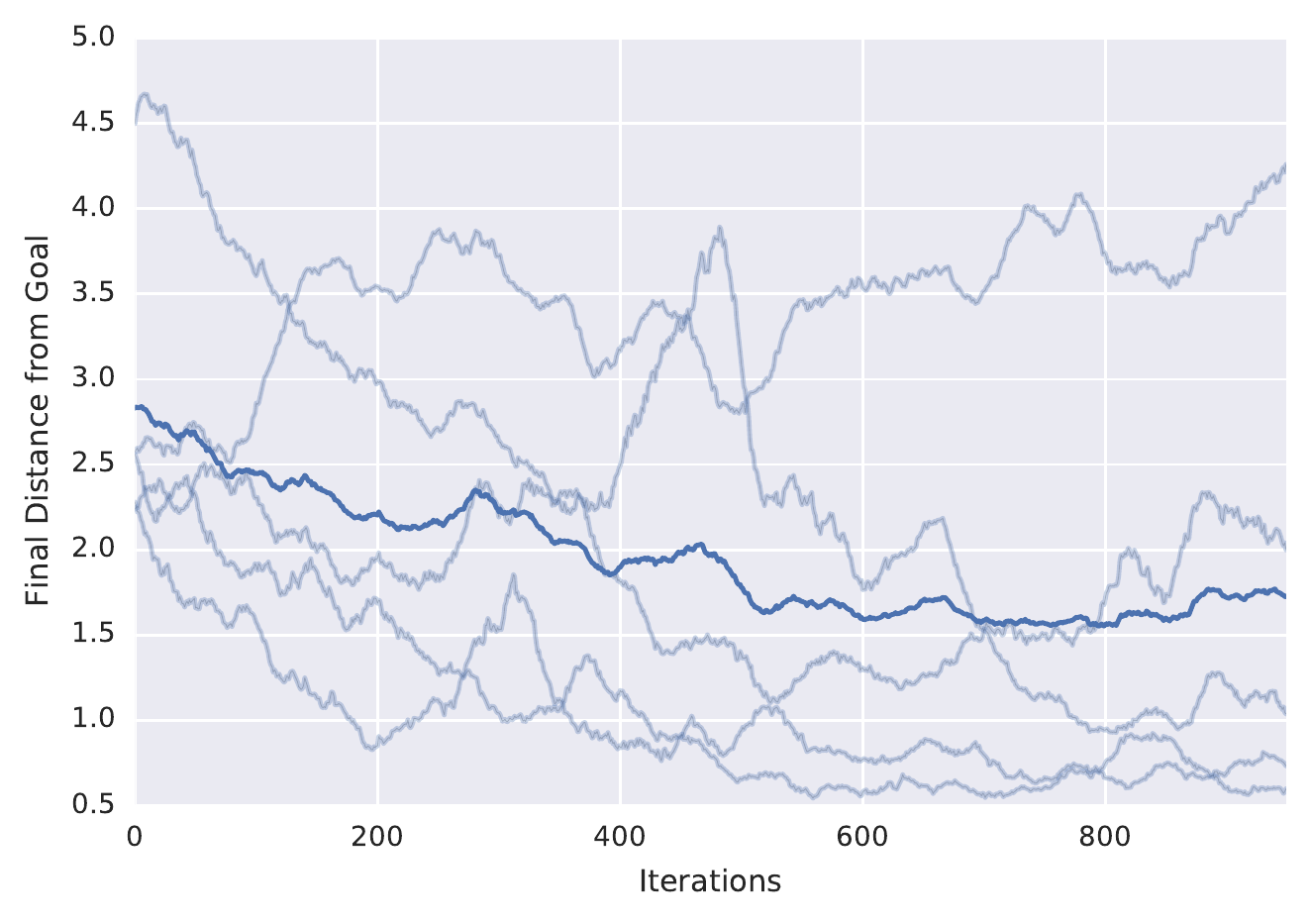}
\caption{Ant - Binary Indicator}
\end{subfigure}
\begin{subfigure}[b]{0.3\textwidth}
\includegraphics[width=\textwidth]{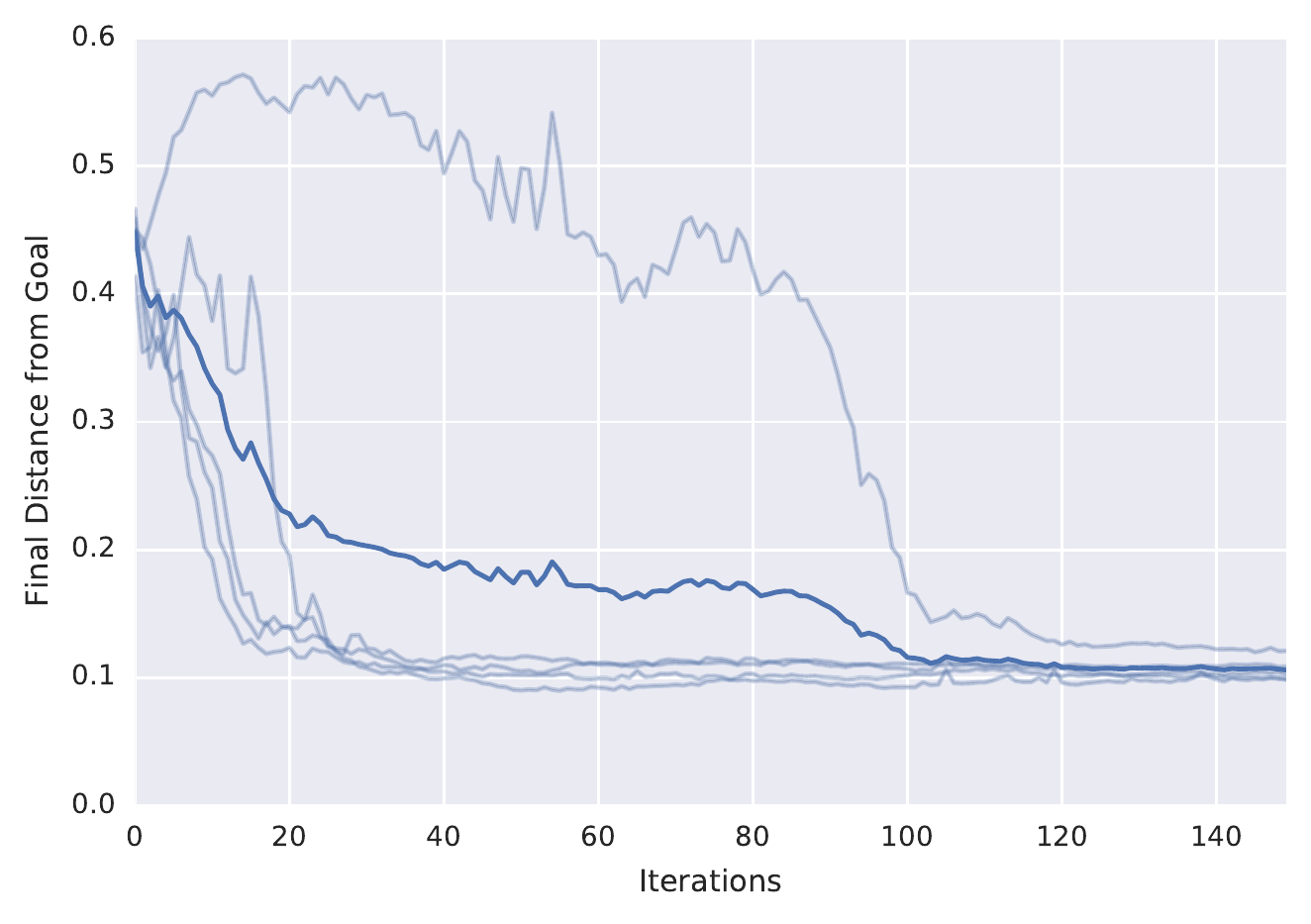}
\caption{Maze - Binary Indicator}
\end{subfigure}
\begin{subfigure}[b]{0.3\textwidth}
\includegraphics[width=\textwidth]{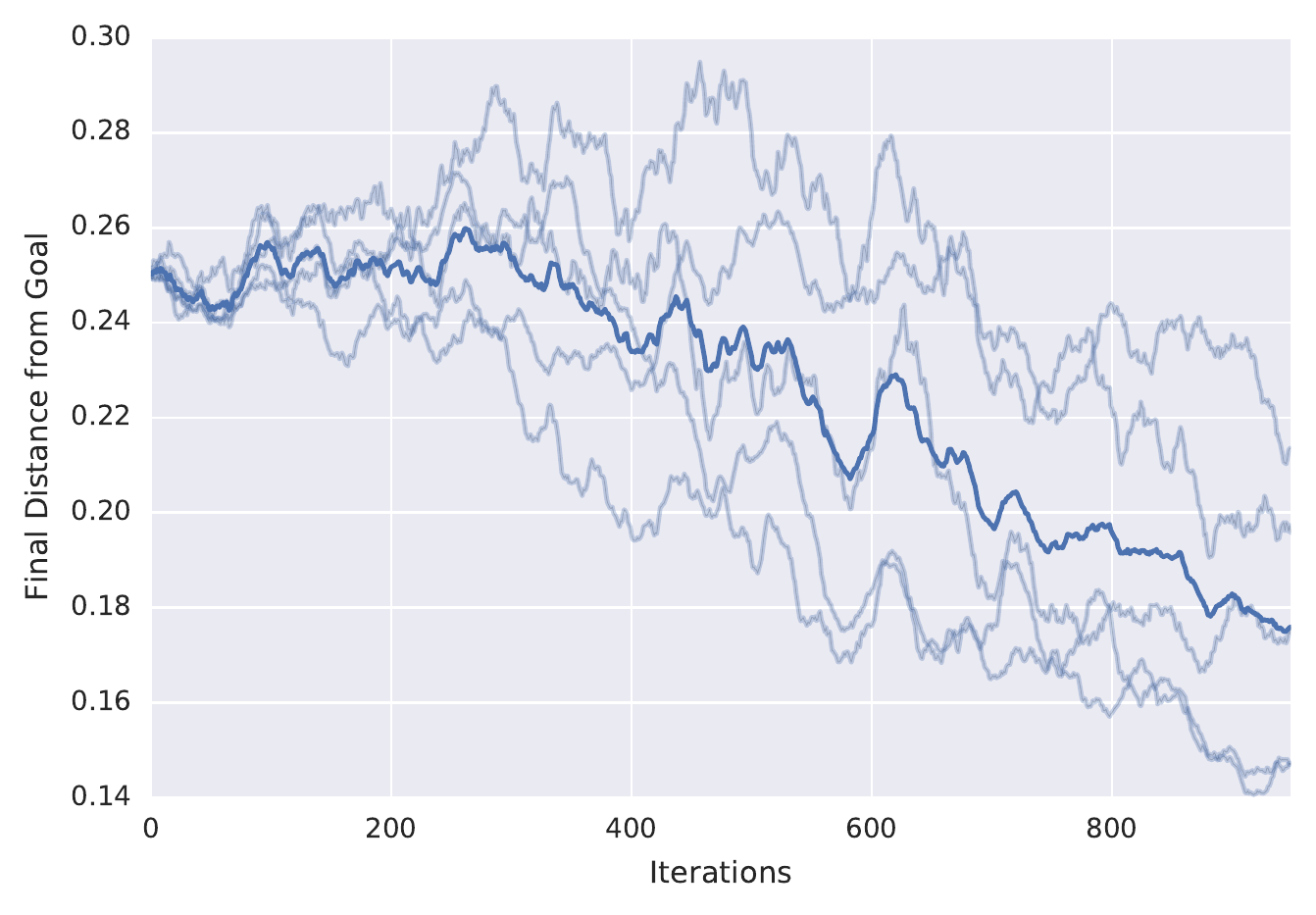}
\caption{Pusher - Binary Indicator}
\end{subfigure}\\
\caption{Learning curves for all methods on each of the five random seeds for the \textit{Ant},\textit{Maze}, and \textit{Pusher} tasks. The mean across the five runs is depicted in bold. }
\end{figure}

\end{document}